\documentclass{aims}
\usepackage{amsmath}
  \usepackage{paralist}
  \usepackage{graphics} 
  \usepackage{epsfig} 
\usepackage{graphicx}  \usepackage{epstopdf}
 \usepackage[colorlinks=true]{hyperref}
 \usepackage[font=footnotesize,labelfont=footnotesize]{caption}
\usepackage[font=scriptsize,labelfont=scriptsize]{subcaption}

\usepackage{algorithm}
\usepackage{algorithmic}

\hypersetup{urlcolor=blue, citecolor=red}

\def\currentvolume{X}

    \def\ppages{X--XX}
     \def\DOI{}

\newtheorem{theorem}{Theorem}[section]
\newtheorem{thm}{Theorem}[section]
\newtheorem{cor}{Corollary}[section]

\theoremstyle{definition}
\newtheorem{defn}[theorem]{Definition}

\newcommand{\NN}{\text{NN}}
\DeclareMathOperator*{\argmax}{arg\,max}
\DeclareMathOperator*{\argmin}{arg\,min}
\newcommand{\Dt}{\mathcal{D}_{t}}
\newcommand{\numclust}{K} 
\newcommand{\Dbtw}{D_{t}^{\text{btw}}}
\newcommand{\Din}{D_{t}^{\text{in}}}
\newcommand{\M}{\mathcal{M}}
\newcommand{\kNN}{k_{\text{NN}}}

\title[Learning by Active Nonlinear Diffusion] 
      {Learning by Active Nonlinear Diffusion}

\author[Mauro Maggioni and James M. Murphy]{}

\subjclass{Primary: 58F15, 58F17; Secondary: 53C35.}
 \keywords{Active learning, statistical learning, diffusion geometry, machine learning, spectral graph theory}

 \email{mauro.maggioni@jhu.edu}
 \email{jm.murphy@tufts.edu}

\thanks{This research is supported by NSF-DMS-125012, NSF-DMS-1724979, NSF-DMS-1708602, NSF-ATD-1737984, AFOSR FA9550-17-1-0280, NSF-IIS-1546392.}

\thanks{$^*$ Corresponding author: James M. Murphy}

\begin{document}
\maketitle

\centerline{\scshape Mauro Maggioni}
\medskip
{\footnotesize
 \centerline{Department of Mathematics, Department of Applied Mathematics and Statistics,}
   \centerline{Mathematical Institute of Data Sciences, Institute of Data Intensive Engineering and Science}
   \centerline{Johns Hopkins University}
   \centerline{Baltimore, MD 21218, USA}
} 

\medskip

\centerline{\scshape James M. Murphy$^*$}
\medskip
{\footnotesize
 \centerline{Department of Mathematics}
   \centerline{Tufts University}
   \centerline{Medford, MA 02155, USA}
}

\bigskip


\begin{abstract}
This article proposes an active learning method for high dimensional data, based on intrinsic data geometries learned through diffusion processes on graphs.  Diffusion distances are used to parametrize low-dimensional structures on the dataset, which allow for high-accuracy labelings of the dataset with only a small number of carefully chosen labels.  The geometric structure of the data suggests regions that have homogeneous labels, as well as regions with high label complexity that should be queried for labels.  The proposed method enjoys theoretical performance guarantees on a general geometric data model, in which clusters corresponding to semantically meaningful classes are permitted to have nonlinear geometries, high ambient dimensionality, and suffer from significant noise and outlier corruption.  The proposed algorithm is implemented in a manner that is quasilinear in the number of unlabeled data points, and exhibits competitive empirical performance on synthetic datasets and real hyperspectral remote sensing images.

\end{abstract}

\section{Introduction}

Statistical and machine learning techniques are revolutionizing the sciences.  Advances in medical diagnosis \cite{Esteva2017_Dermatologist}, automatic game playing \cite{Silver2016_Mastering}, and computer vision \cite{Krizhevsky2012_Imagenet} have been sparked by seismic advances in computational power and innovative learning algorithms and architectures.  However, many state-of-the-art machine learning approaches are predicated on the availability of huge labeled data sets that may be used to train the parameters of the underlying algorithms.  Unfortunately, many important scientific problems do not have large, accurately labeled training sets readily available.  This limits the practicality of many state-of-the-art supervised methods.  Moreover, several fields---medicine and remote sensing, for example---are not amenable to easily generating new labeled data points at scale, due to the high cost of labeling data points.  This renders the applicability of many state-of-the-art supervised learning algorithms---including modern deep learning methods which may depend on millions of parameters---problematic, as generating sufficient training data may be resource-intensive.  

When training datasets do not exist or are burdensome to generate, alternative methods may be used to exploit the glut of unlabeled data.  \emph{Data augmentation} \cite{Tanner1987_Calculation, Van2001_Art} may be used to generate new labeled training points by, for example, perturbing existing training points in a suitable manner.  \emph{Unsupervised methods}---those using no training (labeled) data at all---are ideal when insufficient training data is available, as they work entirely on the unlabeled data.  However, unsupervised methods may be inadequate for highly complex data.  Indeed, such approaches enjoy performance guarantees only when rigid geometrical or statistical properties are made on the data \cite{Arias2011, Arias2017, Schiebinger2015, Mixon2017, Little2017Path, Trillos2019_Geometric}.  Methods that are \emph{semi-supervised} \cite{Chapelle2006_Semi} provide a middle ground between the supervised (abundant labeled data for training) and unsupervised (no labeled data for training) regimes, taking advantage of large quantities of unlabeled data while still allowing labeled points to influence classification.  When the unsupervised structure of the data (e.g. its geometric or statistical properties) are compatible with the labels of the data, semisupervised learning may improve over unsupervised learning and also over classical supervised learning with the same fixed labeled training data.  

This article proposes an \emph{active learning scheme} for high-dimensional datasets exhibiting intrinsically low-dimensional structure.  Active learning is a form of semi-supervised learning in which an algorithm uses the unlabeled data to determine which data points to query for labels.  In the proposed method, the geometry of the data is parametrized through diffusion processes defined on a data-dependent graph \cite{Coifman2005, Coifman2006}, which are robust to high ambient dimensionality, noise, and non-spherical cluster shapes.  The inferred geometry---which is computed without supervision---is then analyzed to determine which data points should be queried for labels; the query points are chosen to have maximum impact, so that relatively few are needed to achieve good empirical performance.  The proposed active learning scheme is called \emph{learning by active nonlinear diffusion (LAND)}.  

\subsection{Major Contributions and Article Outline}

The major contributions of this article are twofold.  First, LAND is proposed and is proven to perform well for data generated according to a flexible geometric data model.  With only a small number of queries, LAND achieves perfect accuracy even for data that is high-dimensional, contains classes that are highly nonlinear or non-compact, and is corrupted by significant noise and outliers.  The theoretical results are derived from an analysis of the underlying diffusion distances, which in turn are amenable to analysis using techniques from spectral graph theory and the analysis of Markov chains.  

Second, the proposed method is implemented numerically.  Taking advantage of fast nearest neighbor search algorithms and eigensolvers for sparse matrices, the proposed method is proven to enjoy \emph{quasilinear complexity} in the number of sample points under the proposed data model, which supposes that the underlying data has intrinsically small dimensionality (in the sense of lying close to a low-dimensional manifold, for example).  LAND is demonstrated on synthetic datasets as well as real hyperspectral images, demonstrating its suitability for high-dimensional geometric data.

The remainder of the article is organized as follows.  Background on active learning and diffusion geometry are presented in Section \ref{sec:Background}.  The geometric data model and algorithm are proposed and analyzed in Section \ref{sec:Algorithms}.  Comparisons with related works are also presented in Section \ref{sec:Algorithms}.  Numerical experiments are in Section \ref{sec:Experiments}.  Conclusions and future research directions are in Section \ref{sec:Conclusions}.

\section{Background}\label{sec:Background}The proposed active learning algorithm exploits the underlying diffusion geometry of data to efficiently determine points to query for labels.  In this section, we review active learning as well as diffusion geometry.  

\subsection{Background on Active Learning}
\label{subsec:BackgroundAL}

Active learning is a type of semisupervised learning in which unlabeled data is analyzed to determine which points to query for labels \cite{Settles2009_Active}.  It differs from traditional semisupervised learning in that the labeling algorithm is permitted to \emph{ask for the labels of certain points}, instead of being provided with a random sample of labeled points.  Under certain data models and methods for parsimoniously selecting query points, the active learning approach can perform as well as traditional semi-supervised or even supervised learning, with far fewer labels \cite{Cohn1994_Improving, Dasgupta2008_General}.  The crucial theoretical question is how to determine which data points should be queried for labels.  The active learning framework assumes there is an underlying budget that can be spent to label points.  This budget should be spent carefully, in order to only query points that are most likely to prove significant for the overall labeling of the data.  

Approaches to active learning may be categorized into two general strategies: hypothesis space reduction and cluster exploitation \cite{Dasgupta2011_Two}.  The first category conceives of supervised learning as a process of using training points to select a ``good" classifier from a large space of possible classifiers.  Asymptotically, as the number of labeled sample points $n_{\ell}\rightarrow\infty$, a consistent supervised learning procedure converges to an optimal classifier.  In practice, the rate of convergence in $n_{\ell}$ is relevant---the faster the rate of convergence, the better the learning algorithm.  In this framework, active learning is a family of methods for selecting query points such that the convergence rate towards a good classifier is fast in $n_{\ell}$, in particular faster than passive sampling methods, for example sampling labels uniformly at random.  That is, query points should be influential in distinguishing between different possible classifiers, and should allow for convergence towards the ``optimal" classifier with fewer points than if the labeled points were selected uniformly at random.  These active learning approaches can, in certain cases, significantly improve the expected error rate of the classifier as a function of $n_{\ell}$ \cite{Balcan2007_Margin, Dasgupta2008_General, Castro2008_Minimax, Balcan2009_Agnostic, Hanneke2011_Rates}.

A second category of active learning approaches seek to exploit cluster structure in the data in order to emphasize sampling near complex regions of the data with heterogeneous labels, and to avoid oversampling near simple, homogeneous regions of the data.  Indeed, if a cluster---detected through a prescribed clustering algorithm---can be estimated as relatively pure with respect to its labels, then it may be efficient to simply give all points in the cluster the same label and to focus the limited querying resources in more ambiguous regions.  A crucial problem in this framework is to tap the budget in a way that balances two different tasks: confirming the label homogeneity of particular data regions and exploring new data regions.  Methods based on iteratively pruning hierarchical clustering trees have been proposed \cite{Dasgupta2008_Hierarchical} and analyzed in terms of label smoothness with respect to the scales of the tree \cite{Urner2013_PLAL}.

The method proposed in this paper is related to the second category, and exploits the underlying geometry of the data sample in order to estimate the most impactful points to query for labels.  In order to develop notions of cluster geometry that are robust to being embedded in a high dimensional space, to being non-spherical in shape, and to corruption by noise and outlier points, the diffusion geometry of the underlying data is estimated and used as the basis for all subsequent pairwise comparisons.  This provides a set of (essentially) geometrically intrinsic coordinates for the data that are robust to dimensionality, nonlinearity, and noise.  

An example of synthetic toy data for which diffusion geometry notably decreases the number of active learning queries necessary for good accuracy appears in Figure \ref{fig:ActiveLearningClustersExample}.  The role of diffusion geometry is crucial to the proposed method, and it is reviewed in detail in Section \ref{subsec:BackgroundDG}.

\begin{figure}[!htb]
	\centering
	\begin{subfigure}[t]{.49\textwidth}
		\includegraphics[width=\textwidth]{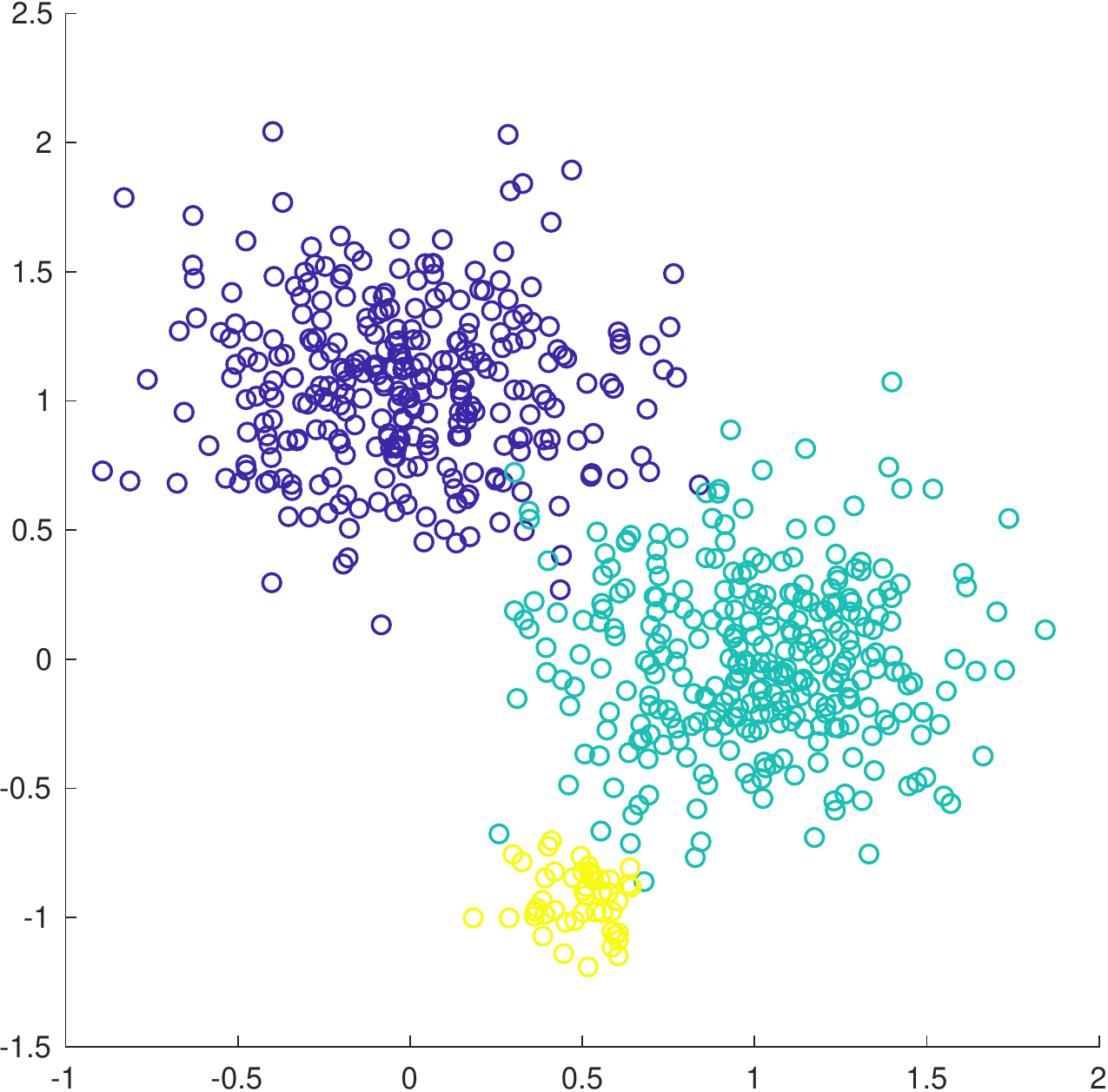}
		\subcaption{Gaussian data}
	\end{subfigure}
	\begin{subfigure}[t]{.49\textwidth}
		\includegraphics[width=\textwidth]{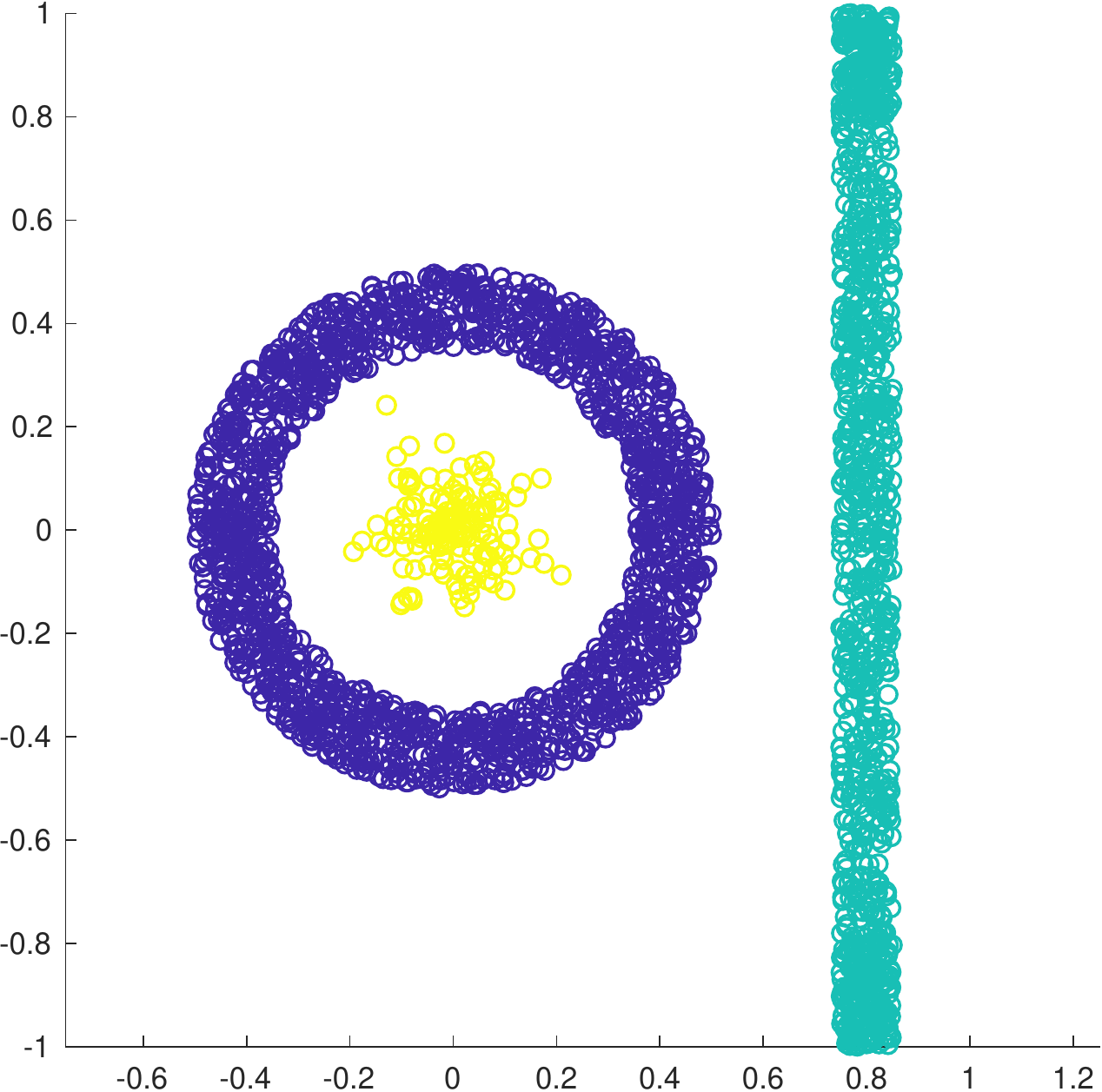}	
		\subcaption{Nonlinear and elongated data}
	\end{subfigure}
		\caption{Data colored by class label.  Both the data in (a) and (b) exhibit cluster structure, which can guide active learning in the case that the labels are constant on these clusters.  Indeed, on the left, using a simple clustering algorithm such as $K$-means suggests that only 3 labels are necessary to correctly label the entire dataset.  For the data on the right, many more than three labels are necessary if $K$-means is used for the underlying clustering, since the clusters are highly elongated and nonlinear.  Indeed, $K$-means will split the annular and elongated clusters.  On the other hand, if pairwise comparisons are made with distances other than Euclidean distances, it may be possible that active learning achieves near perfect results with only 3 labels.  The proposed active learning scheme gains robustness to class shape via diffusion geometry, and is suitable for data in both (a) and (b).}
	\label{fig:ActiveLearningClustersExample}		
\end{figure}

\subsection{Background on Diffusion Geometry}
\label{subsec:BackgroundDG}

Let $X=\{x_{i}\}_{i=1}^{n}\subset\mathbb{R}^{D}$ be discrete data.  The \emph{diffusion geometry} of $X$ is learned through Markov diffusion processes defined on a graph with nodes corresponding to the points $\{x_{i}\}_{i=1}^{n}$ and transition probabilities proportional to the similarities of these points in some metric \cite{Coifman2005, Coifman2006}.  That is, points that are nearby have high probabilities of pairwise transition, and points that are far apart have low probabilities of transition.  By analyzing the diffusion process across time scales, natural geometric structure in the data can be inferred. 

More precisely, let $\mathcal{G}=(X,W)$ be a weighted, undirected graph with nodes $X$ and weight $W_{ij}\in [0,1]$ between $x_{i},x_{j}\in X$.  Typically $W_{ij}=\mathcal{K}(x_{i},x_{j})$ for some symmetric, radial kernel $\mathcal{K}:\mathbb{R}^{D}\times\mathbb{R}^{D}\rightarrow [0,1]$.  The weight matrix $W$ is normalized to produce a Markov transition matrix $P=D^{-1}W$, where $D$ is the diagonal degree matrix with $D_{ii}=\sum_{j=1}^{n}W_{ij}$.  The matrix $P$ is row-stochastic, and \emph{diffusion distances} measure how similar points are according to their transition probabilities in $P$.  

\begin{defn}Let $P$ be a Markov transition matrix defined on $X=\{x_{i}\}_{i=1}^{n}$.  Let $p_{t}(x_{i},x_{j})=(P^{t})_{ij}$.  The \emph{diffusion distance between $x_{i}$ and $x_{j}$ at time $t$ with respect to weight $w:X\rightarrow [0,\infty)$} is $$D_{t}(x_{i},x_{j})=\|p_{t}(x_{i},\cdot)-p_{t}(x_{j},\cdot)\|_{l^{2}(w)}=\sqrt{\sum_{\ell=1}^{n}(p_{t}(x_{i},x_{\ell})-p_{t}(x_{j},x_{\ell}))^{2}w(x_{\ell})}.$$
\end{defn}

The time parameter $t$ is a global time scale at which the diffusion process runs.  For small $t$, the process has run for a short amount of time, which may prevent important, large scale geometric structures in the data from impacting the diffusion distances.  On the other extreme, the diffusion distances all collapse to 0 as $t\rightarrow\infty$, under the assumption that $P$ is ergodic, since $P^{t}$ converges to the rank 1 matrix with the stationary distribution $\pi$ as rows, where $\pi P=\pi$.  When the data has underlying geometric structure, $t$ parametrizes multiscale hierarchy, with small $t$ realizing fine-scale structures and $t$ large realizing coarse-scale structures \cite{Nadler2007_Fundamental, Gavish2013_Normalized, Maggioni2018_Learning}.  

While $P$ is not symmetric, it is diagonally conjugate to a symmetric matrix: $D^{\frac{1}{2}}PD^{-\frac{1}{2}}=D^{-\frac{1}{2}}WD^{-\frac{1}{2}}$.  Hence, $P$ admits a spectral decomposition, which can be exploited for computations of diffusion distance.  More precisely, let $\{(\lambda_{\ell},\phi_{\ell})\}_{\ell=1}^{n}$ be the eigenvalues and eigenvectors of $D^{-\frac{1}{2}}WD^{-\frac{1}{2}}$, sorted so that $1=\lambda_{1}>|\lambda_{2}|\ge\dots\ge |\lambda_{n}|$.  Then $$(P^{t})_{ij}=\sum_{\ell=1}^{n}\lambda_{\ell}^{t}\psi_{\ell}(x_{i})\varphi_{\ell}(x_{j}),$$where $\psi_{\ell}(x_{i})=\phi_{\ell}(x_{i})/\sqrt{\pi(x_{i})}, \varphi_{\ell}(x_{j})=\phi_{\ell}(x_{j})\sqrt{\pi(x_{j})}$.  If $\{\psi_{\ell}\}_{\ell=1}^{n}, \{\varphi_{\ell}\}_{\ell=1}^{n}$ are understood as column vectors, this is equivalent to the decomposition $P^{t}=\sum_{\ell=1}^{n}\lambda_{\ell}^{t}\psi_{\ell}\varphi_{\ell}^{\top}$.  In particular, $\{\varphi_{\ell}\}_{\ell=1}^{n}$ is an orthonormal basis for $l^{2}(1/\pi)$, so that diffusion distances with respect to the weight $w(x_{i})=1/\pi(x_{i})$ may be written in terms of $\{\phi_{\ell}\}_{\ell=1}^{n}$:  $$D_{t}(x_{i},x_{j})=\|p_{t}(x_{i},\cdot)-p_{t}(x_{j},\cdot)\|_{l^{2}(1/\pi)}=\sqrt{\sum_{\ell=1}^{n}\lambda_{\ell}^{2t}(\psi_{\ell}(x_{i})-\psi_{\ell}(x_{j}))^{2}}.$$  If the underlying transition matrix $P$ is approximately low rank, the modulus of the eigenvalues $\{\lambda_{\ell}\}_{\ell=1}^{n}$ decays rapidly, so that for $t$ sufficiently large, this sum may be truncated after $M=O(1)$ eigenpairs yielding the approximate diffusion distances

$$D_{t}(x_{i},x_{j})\approx\sqrt{\sum_{\ell=1}^{M}\lambda_{\ell}^{2t}(\psi_{\ell}(x_{i})-\psi_{\ell}(x_{j}))^{2}}.$$This truncation has the added benefit of denoising the diffusion distances, since the eigenvectors associated with eigenvalues away from 1 in modulus (in some sense the high frequency eigenvectors) correspond not to intrinsic geometric structures in the data, but to random fluctuations produced by sampling \cite{Trillos2018error}.  The embedding $$x_{i}\mapsto (\lambda_{1}^{t}\psi_{1}(x_{i}),\lambda_{2}^{t}\psi_{2}(x_{i}),\dots,\lambda_{M}^{t}\psi_{M}(x_{i}))$$ may be understood as a form of nonlinear dimension reduction, and also as a set of (essentially) geometrically intrinsic coordinates for the data \cite{Lafon2006_Diffusion}.


\section{Proposed Algorithm and Analysis}\label{sec:Algorithms}

Let $\{x_{i}\}_{i=1}^{n}\subset\mathbb{R}^{D}$.  The LAND algorithm requires determining which points should be queried for labels.  This is done by estimating modes of the nonlinear clusters in the data through a combination of density estimation and the diffusion geometry of the data.

Let $p:\mathbb{R}^{D}\rightarrow [0,\infty)$ be a kernel density estimator, for example $$p(x)=\sum_{y\in \NN_{k}(x)}\exp(-\|x-y\|_{2}^{2}/\sigma_{0}^{2}),$$ where $\NN_{k}(x)$ are the $k$-nearest neighbors of $x$ in Euclidean distance, and $\sigma_{0}$ is a scaling parameter.  Let $D_{t}$ be the diffusion distance metric on $X$, and let \begin{align}\label{eqn:rho}\rho_{t}(x)=
\begin{cases}
\min\{D_{t}(x,y) \ | p(y)\ge p(x), x\neq y\}, &x\neq \argmax_{z}p(z)\\
\max_{y\in X} D_{t}(x,y),&x=\argmax_{z}p(z)
\end{cases}
\end{align}be the ($t$-dependent) diffusion distance between a point and its nearest diffusion neighbor of higher density if $x$ is not the maximizer of $p(x)$, and the maximum diffusion distance to another point if $x$ is the maximizer of $p(x)$.  The modes of the data are determined through the quantity $$\Dt(x)=p(x)\rho_{t}(x).$$  Points will have a large $\Dt$ value if they are high density and are $D_{t}$-far from other high density points.  Following \cite{Maggioni2018_Learning}, we characterize the modes of $X$ as the maximizers of $\Dt$.  This notion is robust to data geometry---as captured by diffusion distances---and provides a multiscale hierarchy to the structure of the data.  See Figure \ref{fig:Dt_SyntheticData} for an illustration of how $\Dt$ changes with time.  

\begin{figure}[!htb]
	\centering
	\begin{subfigure}[t]{.49\textwidth}
		\includegraphics[width=\textwidth]{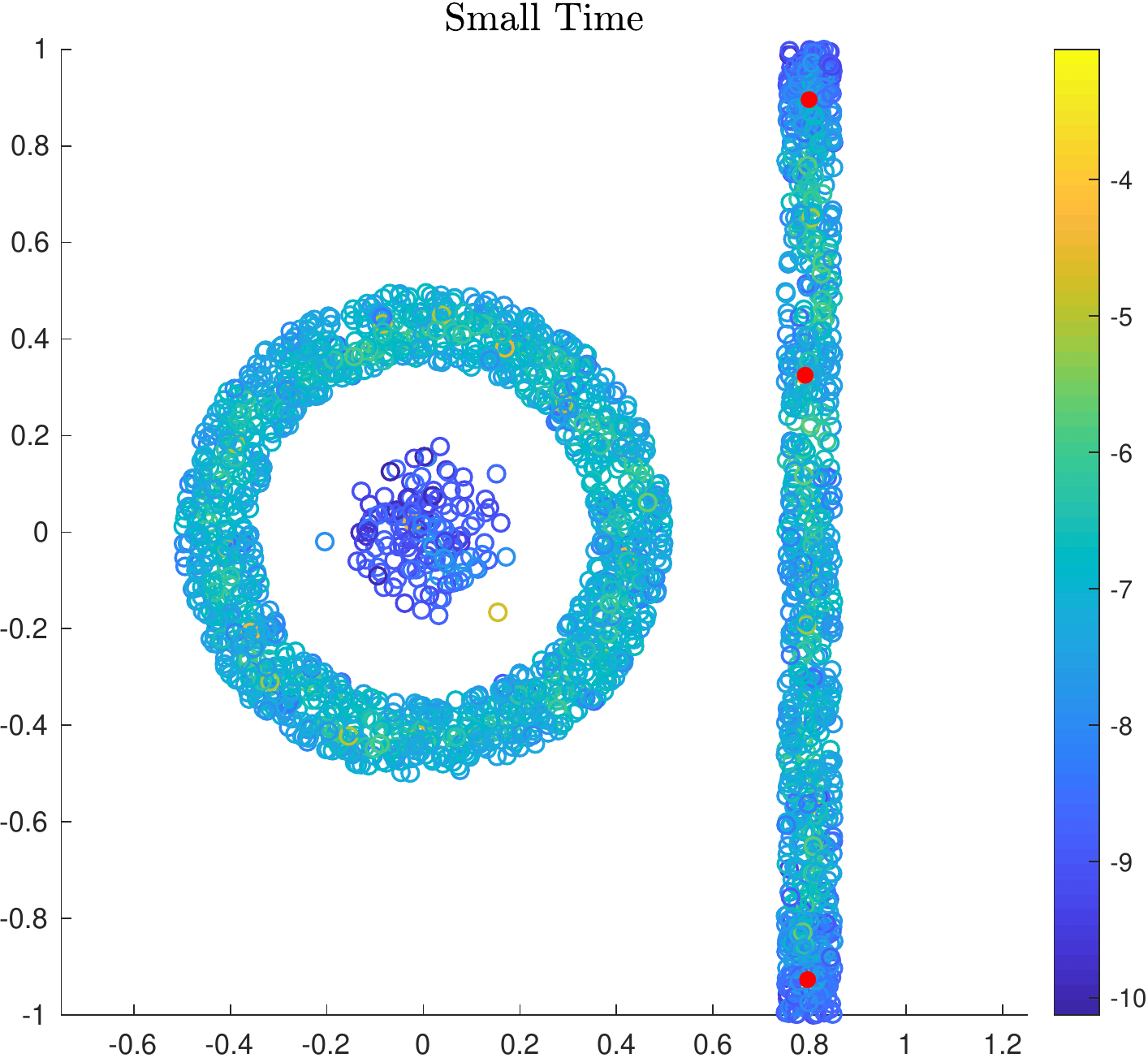}
		\subcaption{Computed $\Dt$ values for $\log_{10}(t)=2$.  The corresponding modes (i.e. the maximizers of $\Dt$) are shown in red.}
	\end{subfigure}
	\begin{subfigure}[t]{.49\textwidth}
		\includegraphics[width=\textwidth]{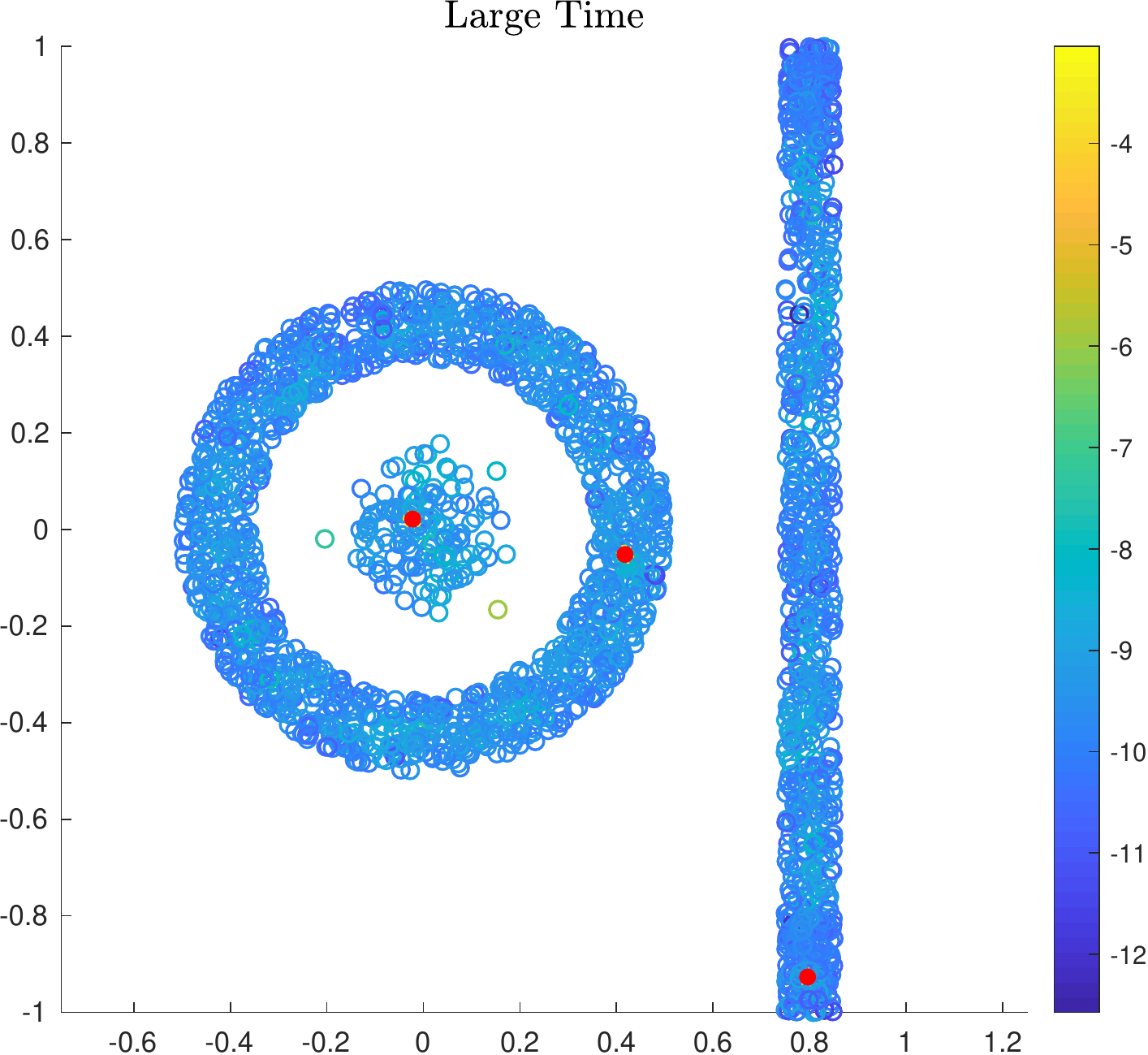}
		\subcaption{Computed $\Dt$ values for $\log_{10}(t)=9$.  The corresponding modes (i.e. the maximizers of $\Dt$) are shown in red.}
	\end{subfigure}
		\begin{subfigure}[t]{.49\textwidth}
		\includegraphics[width=\textwidth]{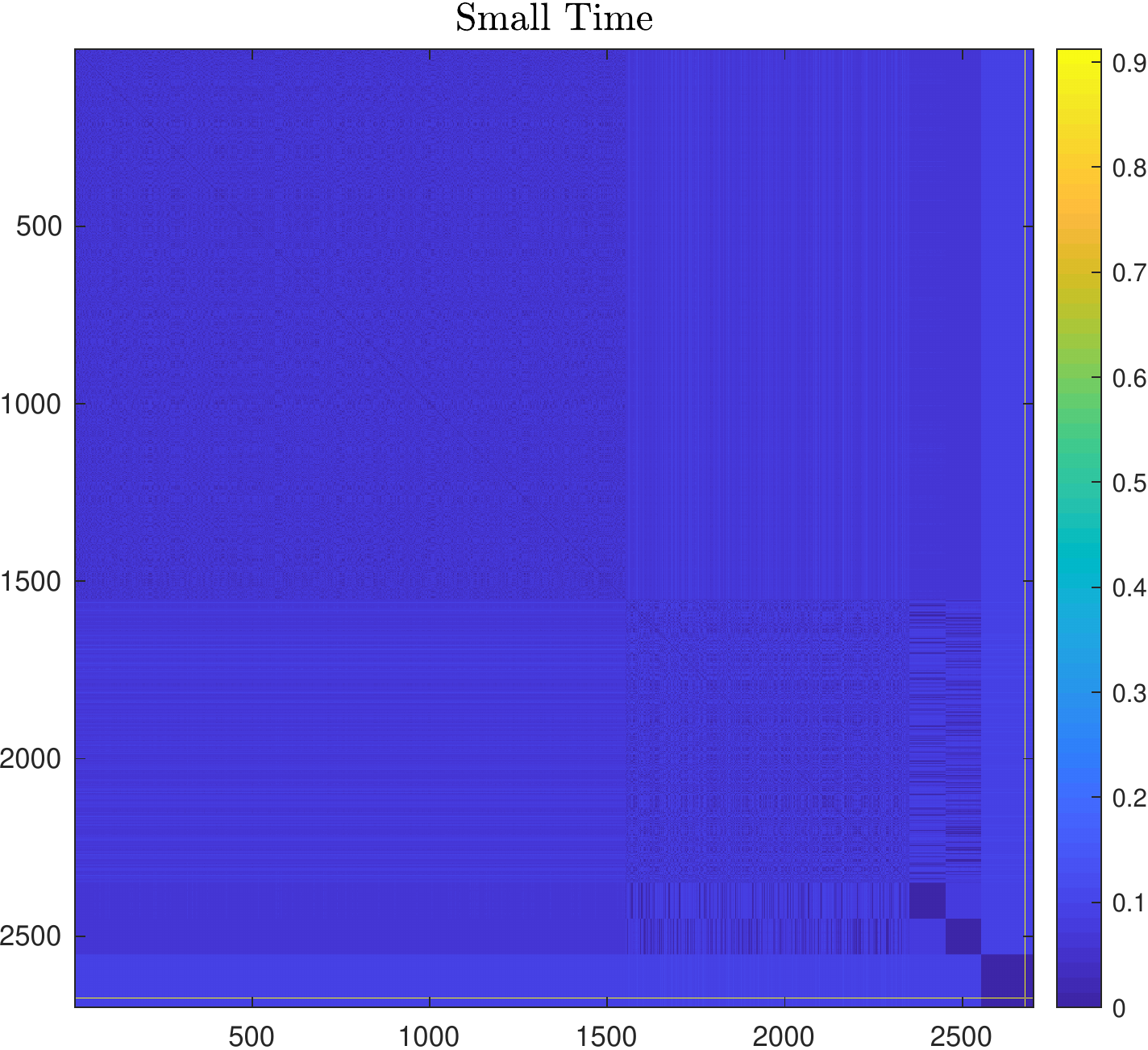}
		\subcaption{$P^{t}$ for $\log_{10}(t)=2$}
	\end{subfigure}
	\begin{subfigure}[t]{.49\textwidth}
		\includegraphics[width=\textwidth]{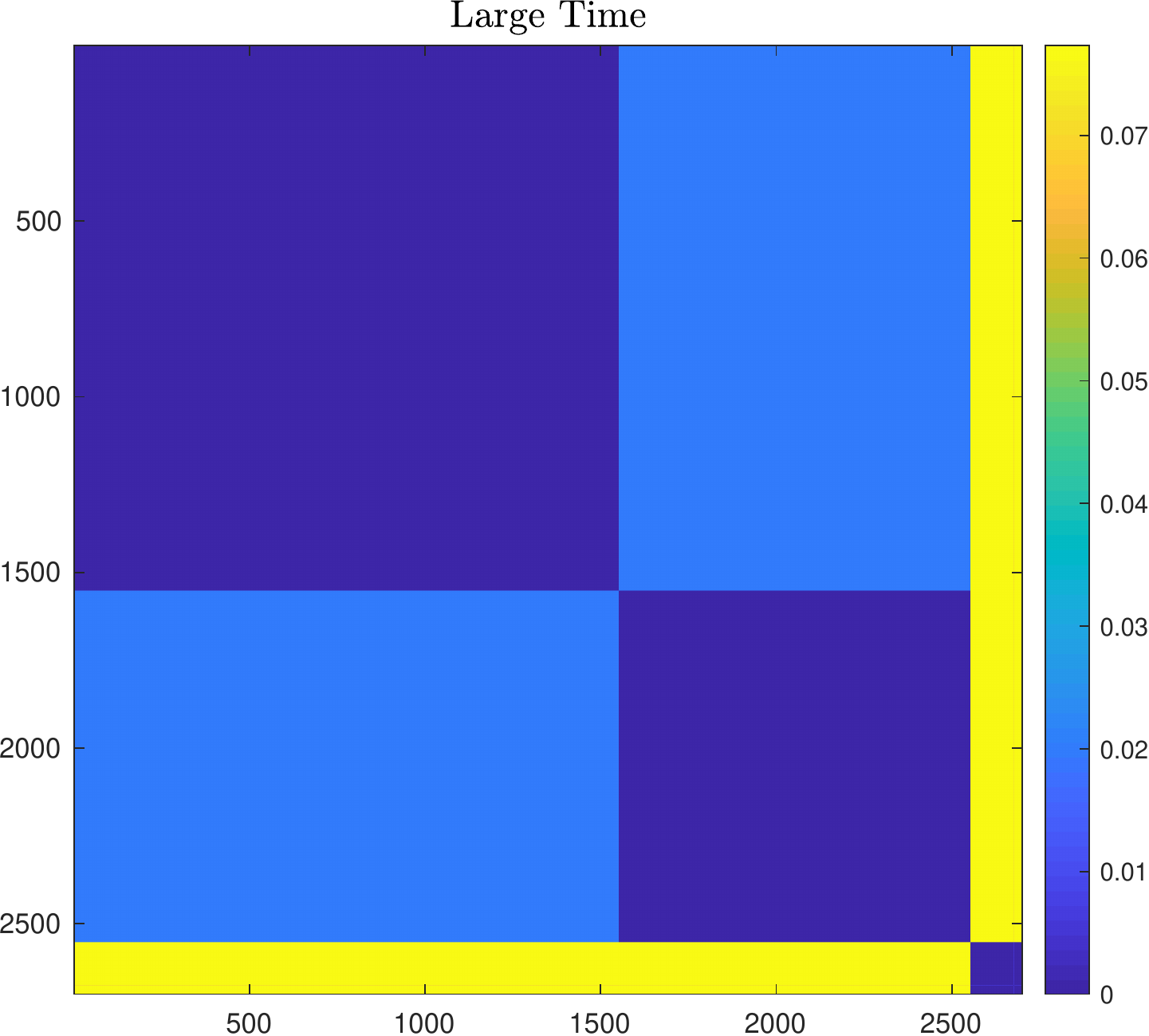}
		\subcaption{$P^{t}$ for $\log_{10}(t)=9$}
	\end{subfigure}
		\caption{In (a) and (b), the values of $\Dt$ are shown for synthetic geometric data from Figure \ref{fig:ActiveLearningClustersExample} (b) for $\log_{10}(t)=2$ and $\log_{10}(t)=9$, respectively.  We see that for small values of $t$, the mode estimation incorrectly places the first three modes on the highly elongated cluster.  For larger time values, the underlying random walk reaches a mesoscopic equilibrium and correct mode estimation is achieved.  The emergence of mesoscopic equilibria is apparent in (c), (d), which show the matrix of diffusion distances at time $\log_{10}(t)=2$ and $\log_{10}(t)=9$, respectively.  When $\log_{10}(t)=2$, $P^{t}$ has not mixed, and there are still substantial within-cluster distances.  For $\log_{10}(t)=9$, $P^{t}$ has reached mesoscopic equilibria, so that within-cluster distances are quite small, yet intra-cluster distances are still small \cite{Maggioni2018_Learning}.}
	\label{fig:Dt_SyntheticData}		
\end{figure}

\subsection{Learning by Unsupervised Nonlinear Diffusion}

In \cite{Maggioni2018_Learning}, the maximizers of $\Dt$ were proposed as cluster modes, and diffusion distances and density were used to label all other points relative to these modes.  We summarize this unsupervised learning algorithm, called \emph{learning by unsupervised nonlinear diffusion (LUND)} in Algorithm \ref{alg:LUND}.  This algorithm was proven to perfectly cluster certain data, for an appropriate choice of time parameter $t$, and is robust to non-spherical data geometries and cluster overlap. 

\begin{algorithm}[!htb]
	\caption{\label{alg:LUND}Learning by Unsupervised Nonlinear Diffusion (LUND)}
	\flushleft
	\textbf{Input:}\\
	\begin{itemize}
	\item $\{x_{i}\}_{i=1}^{n}$ (Unlabeled Data)\\
	 \item $\{(\lambda_{\ell},\psi_{\ell})\}_{\ell=1}^{M}$ (Spectral Decomposition of $P$)\\
	 \item $\{p(x_{i})\}_{i=1}^{n}$ (Empirical Density Estimate)\\
	 \item $\{\rho_{t}(x_{i})\}_{i=1}^{n}$ (\ref{eqn:rho})\\
	 \item $t$ (Time Parameter)\\
	  \end{itemize}
	
	\textbf{Output:}\\
	\begin{itemize}
	\item $\hat{K}$ (Estimated Number of Clusters)\\
	\item $Y$ (Labels)\\
	\end{itemize}
	\begin{algorithmic}[1]
	\STATE Compute $\Dt(x_{i})=p(x_{i})\rho_{t}(x_{i})$.  
	\STATE Sort the data in order of decreasing $\Dt$ value to acquire the ordering $\{x_{m_{i}}\}_{i=1}^{n}$.\STATE Estimate $\hat{K}=\argmax_{i}\left(\Dt(x_{m_{i}})/\Dt(x_{m_{i+1}})\right).$
	\FOR{$k=1:\hat{\numclust}$}
	\STATE $Y(x_{m_{k}})=k$.
	\ENDFOR
	\STATE Sort $X$ according to $p(x)$ in decreasing order as $\{x_{\ell_{i}}\}_{i=1}^{n}$.  
	\FOR{$i=1:n$}
	\IF{$Y(x_{\ell_{i}})=0$}
	\STATE $Y(x_{\ell_{i}})=Y(z_{i}), \, z_{i}=\displaystyle\argmin_{z}\{D_{t}(z,x_{\ell_{i}})\ | \ p(z)>p(x_{\ell_{i}}) \text{ and } Y(z_{i})>0\}$.
	\ENDIF
	\ENDFOR
	\end{algorithmic}
\end{algorithm}

It was shown that, depending on the well-connectedness of the clusters compared to their separations, the range of $t$ for which Algorithm \ref{alg:LUND} performs well may be large \cite{Maggioni2018_Learning}.  However, developing methods for estimating an appropriate choice of $t$ without using any labeled data is an important and only partially addressed problem.  Indeed, if the data admits hierarchical cluster structure, then several choices of $t$ may be appropriate, leading to different reasonable clusterings.  In this context, querying a small number of points for labels can disambiguate between these different clusterings.

\subsection{Learning by Active Nonlinear Diffusion}

In the active learning setting, we characterize potential classes as being composed of $D_{t}$-orbits around the maximizers of $\Dt$.  These orbits partition the data, and are comparable to elements of a Voronoi tessellation \cite{Aurenhammer1991_Voronoi}.  In the case that the labels for the data are smooth with respect to this partition, querying the maximizers of $\Dt$ is a more efficient use of a sampling budget than uniform random sampling.  The proposed algorithm, denoted \emph{Learning by Active Nonlinear Diffusion (LAND)} appears in Algorithm \ref{alg:LAND}.

\begin{algorithm}[!htb]
	\caption{\label{alg:LAND}Learning by Active Nonlinear Diffusion (LAND)}
	\flushleft
	\textbf{Input:} 
	\begin{itemize}
	\item $\{x_{i}\}_{i=1}^{n}$ (Unlabeled Data)\\
	 \item $\{(\lambda_{\ell},\psi_{\ell})\}_{\ell=1}^{M}$ (Spectral Decomposition of $P$)\\
	 \item $\{p(x_{i})\}_{i=1}^{n}$ (Kernel Density Estimate)\\
	 \item $\{\rho_{t}(x_{i})\}_{i=1}^{n}$ (\ref{eqn:rho})\\
	 \item $t$ (Time Parameter)\\
	 \item $B$ (Budget)\\
	 \item $\mathcal{O}$ (Labeling Oracle)\\
	  \end{itemize}
	\textbf{Output:} 
	\begin{itemize}
	\item $Y$ (Labels)
	\end{itemize}
	\begin{algorithmic}[1]
	\STATE Compute $\Dt(x_{i})=p(x_{i})\rho_{t}(x_{i})$.  
	\STATE Sort the data in order of decreasing $\Dt$ value to acquire the ordering $\{x_{m_{i}}\}_{i=1}^{n}$.\FOR{$i=1:B$}
	\STATE Query $\mathcal{O}$ for the label $L(x_{m_{i}})$ of $x_{m_{i}}$.
	\STATE Set $Y(x_{m_{i}})=L(x_{m_{i}})$.
	\ENDFOR
	\STATE Sort $X$ according to $p(x)$ in decreasing order as $\{x_{\ell_{i}}\}_{i=1}^{n}$.  
	\FOR{$i=1:n$}
	\IF{$Y(x_{\ell_{i}})=0$}
	\STATE $Y(x_{\ell_{i}})=Y(z_{i}), \, z_{i}=\displaystyle\argmin_{z}\{D_{t}(z,x_{\ell_{i}})\ | \ p(z)>p(x_{\ell_{i}}) \text{ and } Y(z)>0\}$.
	\ENDIF
	\ENDFOR
	\end{algorithmic}
\end{algorithm}

\subsubsection{Analysis of LAND}

From a theoretical standpoint, it is of interest to know when querying a small number of points (Algorithm \ref{alg:LAND}) offers substantial be benefit compared to unsupervised learning (Algorithm \ref{alg:LUND}).  Suppose that the underlying data consists of distinct classes $X=\bigcup_{k=1}^{\numclust} X_{k}$, with all points in $X_{k}$ having label $k$.  Let \begin{align*}\Din&=\max_{k=1,\dots,\numclust}\max_{x,y\in X_{k}}D_{t}(x,y),\\ 
\Dbtw&=\min_{k\neq k'}\min_{x\in X_{k},y\in X_{k'}}D_{t}(x,y)\end{align*} be the maximum within-class and minimum between-class diffusion distances at time $t$, respectively.    Let $$\mathcal{M}=\{x\in X \ | \ \exists k \text{ such that }  x = \argmax_{y\in X_{k}}p(y)\},$$ $\max(\M)=\max_{x\in\M}p(x), \min(\M)=\min_{x\in\M}p(x)$ be the density maximizers of the distinct classes, the maximum density among such classwise maximizers, and the minimum density among the classwise maximizers, respectively.  In \cite{Maggioni2018_Learning}, it is shown that if \begin{align}\label{eqn:LUND_Condition}\Din/\Dbtw<\max(\M)/\min(\M),\end{align}then the data can be labeled in a fully unsupervised manner by Algorithm \ref{alg:LUND}.  However, the underlying density conditions may not be satisfied in practice, particularly if there are strong discrepancies between the density of the most dense point in each cluster.  Moreover, (\ref{eqn:LUND_Condition}) depends strongly on $t$.  Introducing the active learning scheme allows to bypass this potentially stringent density condition and still achieve perfect accuracy, at the cost of querying the labels of a small number of points.

\begin{thm}\label{thm:AccuracyGuarantee}Let $X=\bigcup_{k=1}^{\numclust}$ be data to classify.  Suppose that $\Din<\Dbtw$, and that the $B$ maximizers of $\Dt$ include the elements of $\M$.  Then LAND with a budget of size $B$ achieves perfect classification accuracy.
\end{thm}

\begin{proof}

If the $B$ maximizers of $\Dt$ include all the density maximizers of the distinct classes, that is, the elements of $\mathcal{M}$, then the LAND queries guarantee these points are all labeled correctly.  Then the result follows by induction on the data points sorted in order of decreasing $p(x)$ value.  Indeed, for an unlabeled point $x\in X_{k}$, its nearest diffusion neighbor of higher density, $x^{*}$, must be in the same class $X_{k}$, since $\Din<\Dbtw$.  Moreover, that point is already labeled as $Y(x^{*})=k$, since $p(x^{*})\ge p(x)$.  Hence, $Y(x)=k$.  

\end{proof}

Theorem \ref{thm:AccuracyGuarantee} asserts that the LAND algorithm achieves perfect accuracy as long as $\Din<\Dbtw$ and $B$ is large enough so that all elements of $\M$ are among the $B$ maximizers of $\Dt$.  Compared to LUND, LAND does not require that $\Din/\Dbtw<\max(\M)/\min(\M)$ to guarantee strong performance.  This is an important point in practice, since the density between different regions of the data may vary considerably.  Ultimately, active learning is most useful when the budget $B$ may be taken very small compared to $n$; we shown in Section \ref{sec:Experiments} that even a budget of just a few points may significantly improve accuracy on synthetic and real datasets.

\subsection{Comparison with Related Methods}  It is natural to compare LAND with related cluster-based active learning methods, as well as its unsupervised variant LUND.

\subsubsection{Comparisons with Related Active Learning Methods}

As discussed in Section \ref{subsec:BackgroundAL}, active learning methods may be categorized as falling into two broad classes: those based on refining the hypothesis space of classifiers, and those based on exploiting cluster structure in the data.  LAND falls into the second category; it is thus natural to compare it with existing cluster-based active learning algorithms.  

Many active learning algorithms that exploit cluster structure in the data proceed by constructing a hierarchical clustering on the data, often represented in the form of a dendrogram \cite{Friedman2001}.  Given such a structure, sample queries are made in order to explore heterogeneous regions of the tree (leaves with highly mixed labels) and to avoid sampling from homogeneous regions of the data (leaves that consist mostly of a single class).  The key challenge is to balance the cost of exploring ambiguous regions of the data with establishing the homogeneity of other regions.  

Efficient algorithms that are statistically consistent have been proposed \cite{Dasgupta2008_Hierarchical} and analyzed using the notion of ``probabilistic Lipschitzness," which quantifies purity of leaves of the hierarchical clustering \cite{Urner2013_PLAL}.  These approaches make analyzing the hierarchical tree the central problem; the problem of whether or not a particular method for constructing a hierarchical tree is appropriate or not is not directly considered.  Indeed, it is common to construct the underlying hierarchical tree with standard methods, for example average-linkage clustering \cite{Dasgupta2008_Hierarchical} or single linkage clustering \cite{Friedman2001}.  Despute their pervasiveness, these methods for constructing hierarchical trees suffer from a lack of robustness to pernicious chains in the data (single-linkage) and geometric distortion (average-linkage).  Active learning based on hierarchical trees performs well when the leaves of the tree become pure quickly when descending from the root node; if the underlying tree does not exhibit pure leaves until relatively deep in the tree, many samples are required for active learning, and the method may not improve substantially over random sampling.  

Unlike average linkage and single linkage clustering, the proposed LAND method explicitly incorporates the underlying geometry of the data to construct clusters of multiscale granularity, which can then be exploited for active querying.  The LAND algorithm may be interpreted as a method for constructing the underlying hierarchical tree, which has the desirable property that the leaves are essentially robust to geometric transformations of the underlying clusters, i.e. to making the clusters elongated or nonlinear.  Indeed, given a number of clusters $\numclust$, one can run a variant of LUND in which $\numclust$ is input as a parameter; see Algorithm \ref{alg:LUND_K_Known}.  

\begin{algorithm}[!htb]
	\caption{\label{alg:LUND_K_Known}LUND, $\numclust$ Known}
	\flushleft
	\textbf{Input:}\\
	\begin{itemize}
	\item $\{x_{i}\}_{i=1}^{n}$ (Unlabeled Data)\\
	 \item $\{(\lambda_{\ell},\psi_{\ell})\}_{\ell=1}^{M}$ (Spectral Decomposition of $P$)\\
	 \item $\{p(x_{i})\}_{i=1}^{n}$ (Empirical Density Estimate)\\
	 \item $\{\rho_{t}(x_{i})\}_{i=1}^{n}$ (\ref{eqn:rho})\\
	 \item $t$ (Time Parameter)\\
	 \item $K$ (Number of Clusters)\\
	  \end{itemize}
	
	\textbf{Output:}\\
	\begin{itemize}
	\item $Y$ (Labels)\\
	\end{itemize}
	\begin{algorithmic}[1]
	\STATE Compute $\Dt(x_{i})=p(x_{i})\rho_{t}(x_{i})$.  
	\STATE Sort the data in order of decreasing $\Dt$ value to acquire the ordering $\{x_{m_{i}}\}_{i=1}^{n}$.\FOR{$k=1:\numclust$}
	\STATE $Y(x_{m_{k}})=k$.
	\ENDFOR
	\STATE Sort $X$ according to $p(x)$ in decreasing order as $\{x_{\ell_{i}}\}_{i=1}^{n}$.  
	\FOR{$i=1:n$}
	\IF{$Y(x_{\ell_{i}})=0$}
	\STATE $Y(x_{\ell_{i}})=Y(z_{i}), \, z_{i}=\displaystyle\argmin_{z}\{D_{t}(z,x_{\ell_{i}})\ | \ p(z)>p(x_{\ell_{i}}) \text{ and } Y(z)>0\}$.
	\ENDIF
	\ENDFOR
	\end{algorithmic}

\end{algorithm}

It is then natural to compare the purity of the nodes of a hierarchical tree at scale $\numclust$, with the purity of the clusters learned by Algorithm \ref{alg:LUND_K_Known} with number of clusters equal to $\numclust$.  More generally, let $\mathcal{C}=\{C_{k}\}_{k=1}^{\numclust}$ be a clustering of labeled data $\{(x_{i},y_{i})\}_{i=1}^{n}$.  Let $\bar{y}_{k}$ be the most common label among the points in $C_{k}$.  The \emph{purity of the clustering} $\mathcal{C}$ is defined as $$\mathcal{P}(\mathcal{C})=\frac{1}{n}\sum_{k=1}^{\numclust}|\{x_{i}\in C_{k}\ | \ y_{i}=\bar{y}_{k}\}|.$$Given a hierarchical clustering $\{\mathcal{C}_{\ell}\}_{\ell=1}^{n}$---that is, $\mathcal{C}_{1}$  consists of 1 cluster with all points, $\mathcal{C}_{n}$ consists of $n$ singleton clusters, and $\mathcal{C}_{\ell+1}$ is the same clustering as $\mathcal{C}_{\ell}$, but with two of the clusters split--- the purity of the clustering at the $\ell^{th}$ scale is $\mathcal{P}(\mathcal{C}_{\ell})$.  Clearly $\mathcal{P}(\mathcal{C}_{\ell})$ is non-decreasing as a function of $\ell$, and $\mathcal{P}(\mathcal{C}_{n})=1.$  If the growth of $\mathcal{P}(\mathcal{C}_{\ell})$ towards 1 is rapid in $\ell$, then an active sampler does not need to search deeply into the tree to find regions with homogeneous labels.  In Figure \ref{fig:HierarchicalPurities}, a plot of $\mathcal{P}(\mathcal{C}_{\ell})$ is shown for three synthetic datasets with three different families of clusterings: single linkage clusters, average linkage clusters, and the clusters learned by Algorithm \ref{alg:LUND_K_Known}.

\begin{figure}[!htb]
	\centering
	\begin{subfigure}[t]{.32\textwidth}
		\includegraphics[width=\textwidth]{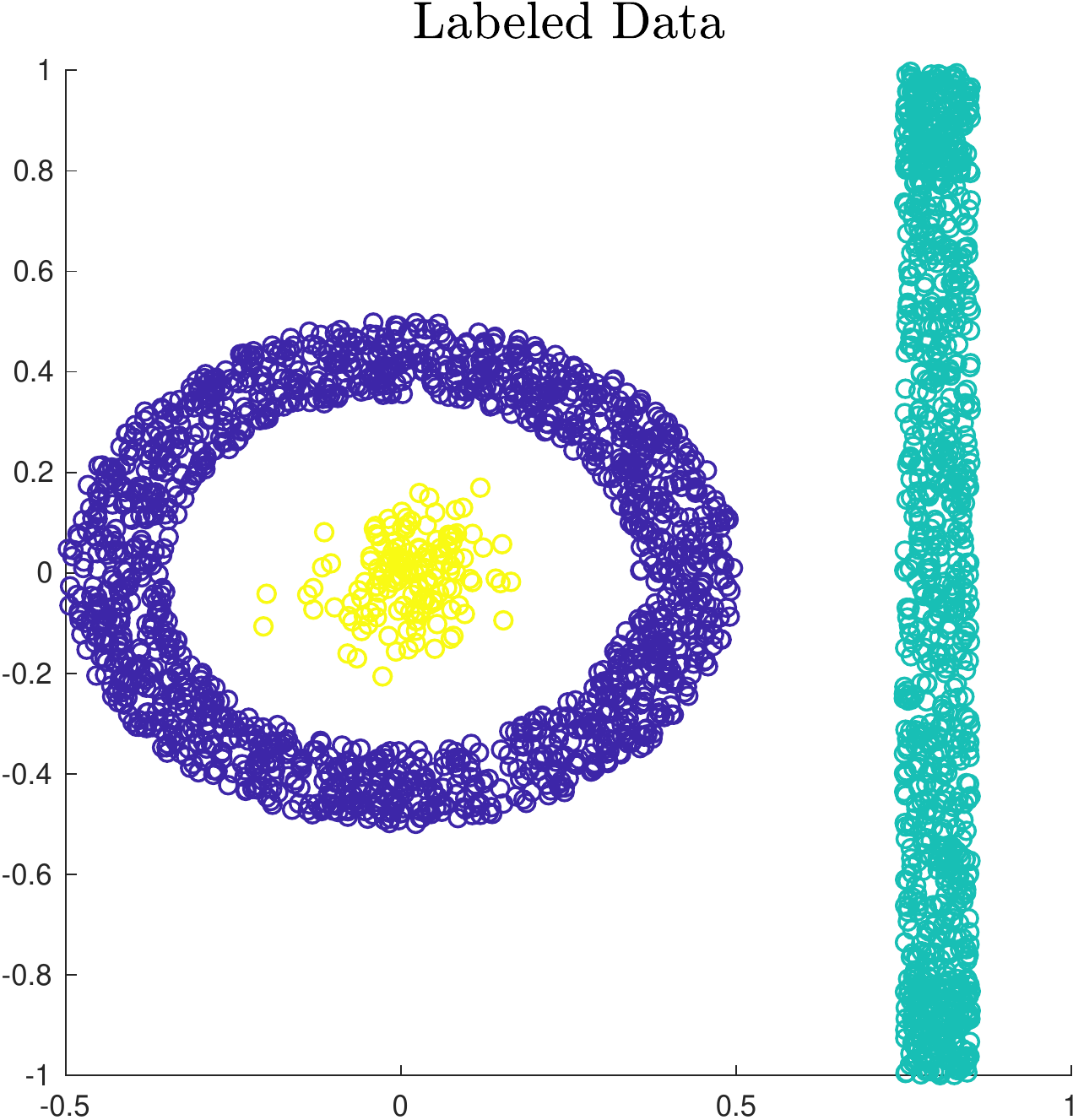}
		\subcaption{Geometric data}
	\end{subfigure}
	\begin{subfigure}[t]{.32\textwidth}
		\includegraphics[width=\textwidth]{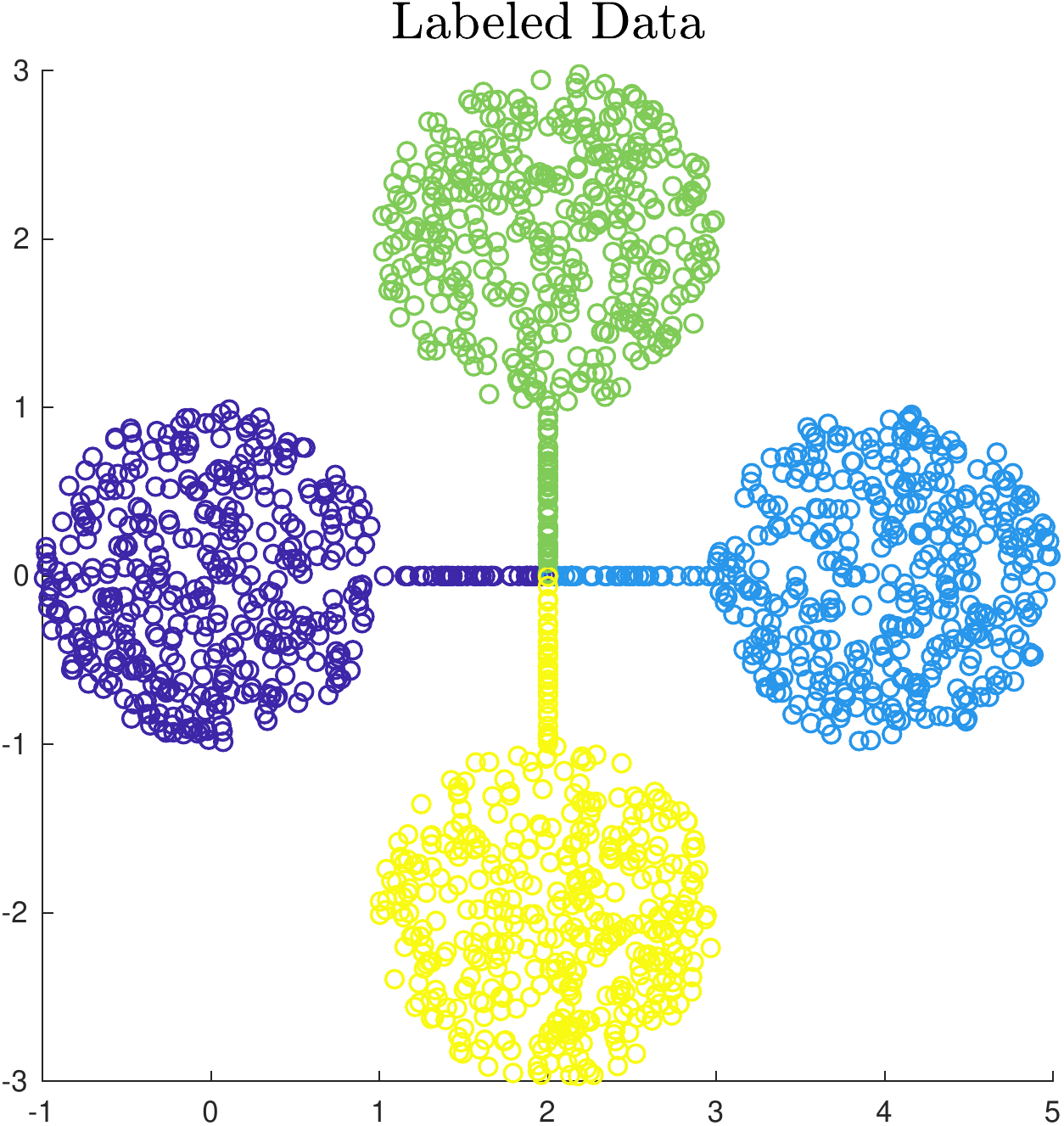}
		\subcaption{Bottleneck data}
	\end{subfigure}
	\begin{subfigure}[t]{.32\textwidth}
		\includegraphics[width=\textwidth]{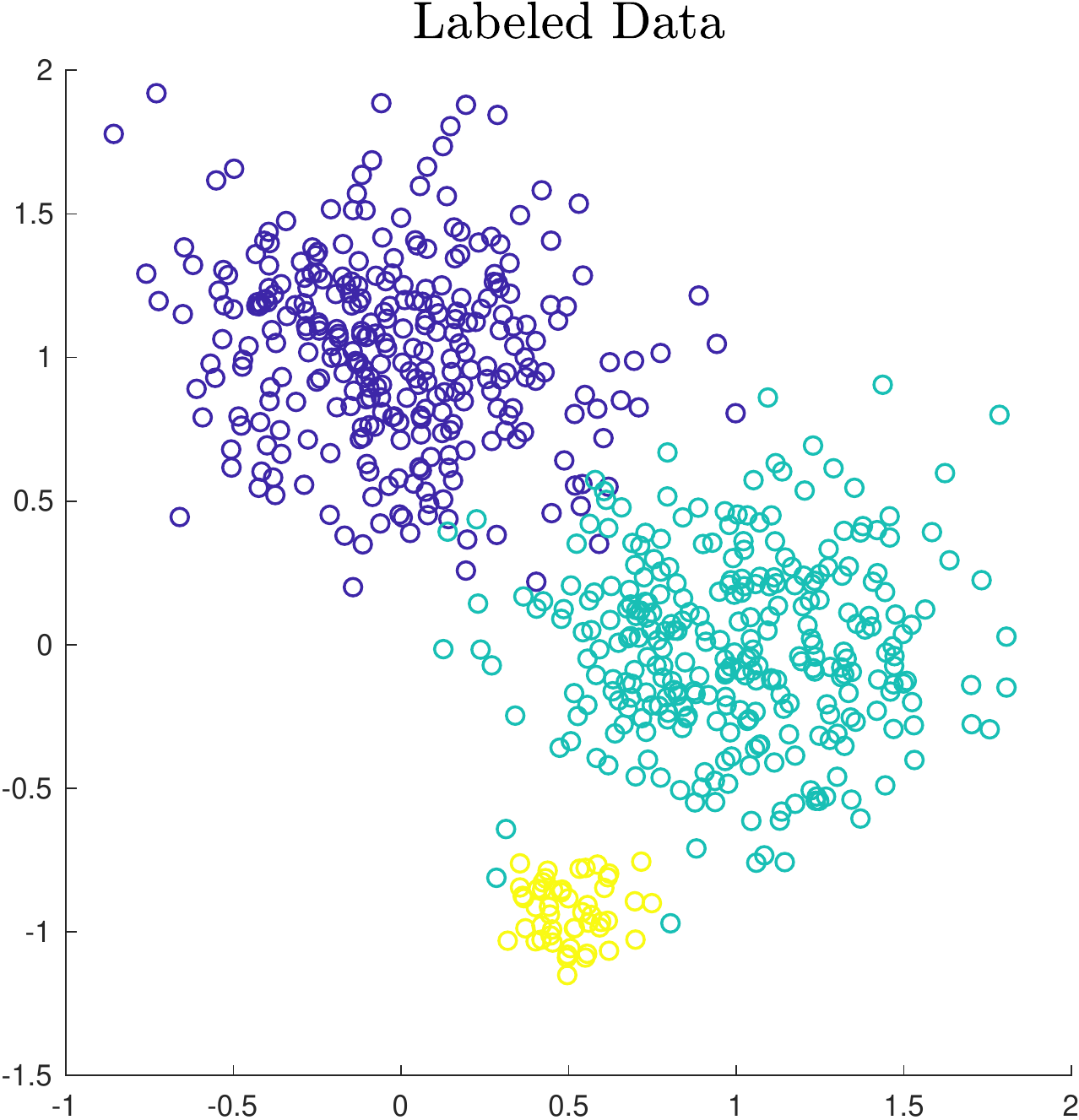}
		\subcaption{Gaussian data}
	\end{subfigure}
	\begin{subfigure}[t]{.32\textwidth}
		\includegraphics[width=\textwidth]{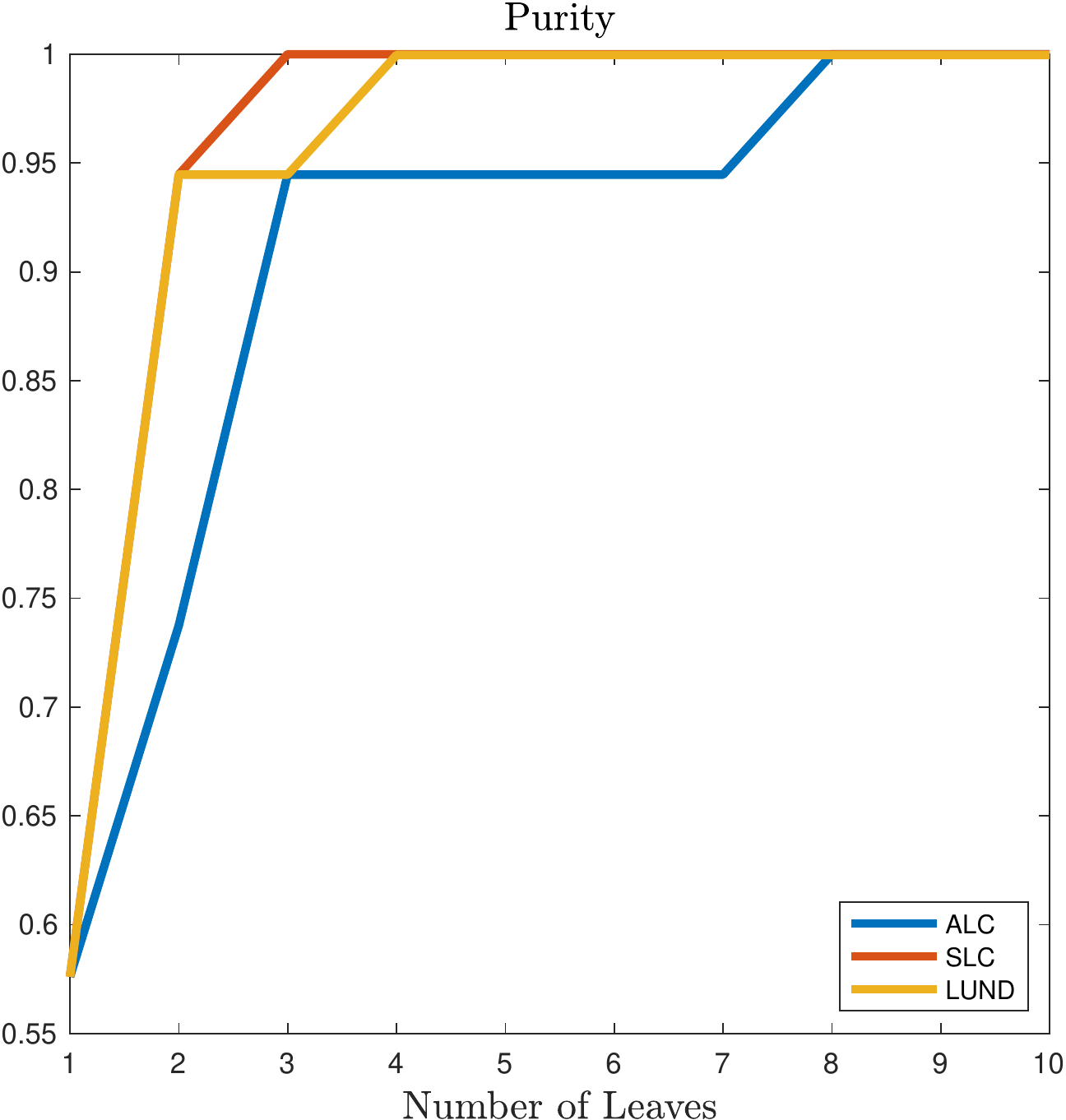}
		\subcaption{Hierarchical purity of geometric data}
	\end{subfigure}
	\begin{subfigure}[t]{.32\textwidth}
		\includegraphics[width=\textwidth]{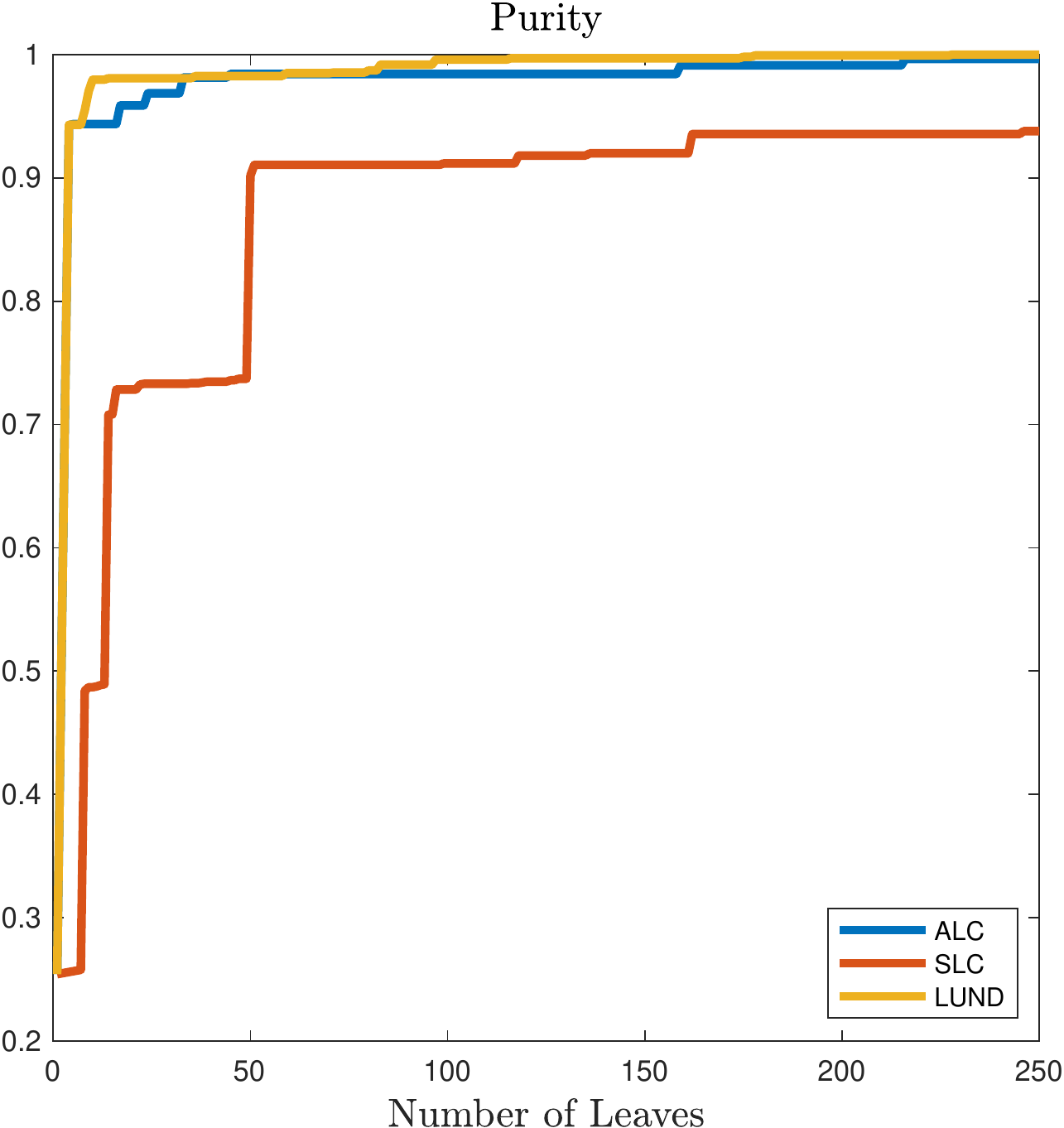}
		\subcaption{Hierarchical purity of bottleneck data}
	\end{subfigure}
	\begin{subfigure}[t]{.32\textwidth}
		\includegraphics[width=\textwidth]{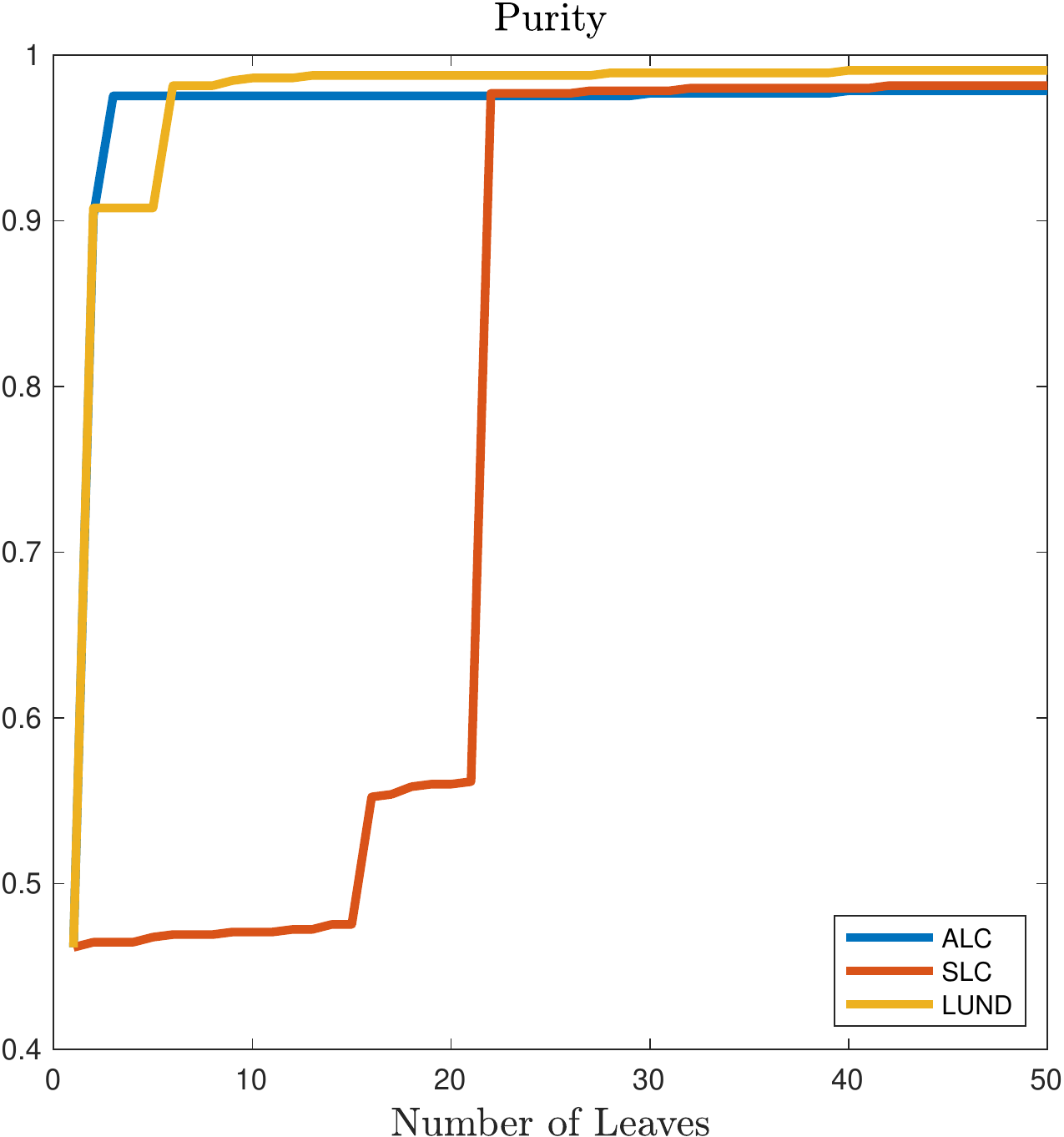}
		\subcaption{Hierarchical purity of Gaussian data}
	\end{subfigure}
		\caption{\emph{Top row}:  Three different synthetic datasets in two dimensions are shown, categorized as geometric, bottleneck and Gaussian.  \emph{Bottom row:} Plots of node purity for three different multiscale, hierarchical methods of clustering: average linkage clustering (ALC), single linkage clustering (SLC) and learning by unsupervised nonlinear diffusion (LUND).  As the number of leaves/clusters increases, purity is non-decreasing.  The purity of the LUND clusters converges more rapidly to the optimal value 1, indicating that high accuracy can be gained by correctly labeling a smaller number of clusters, compared to ALC and SLC.}
	\label{fig:HierarchicalPurities}		
\end{figure}

We see that for the geometric data, the clusters learned from average linkage clustering achieve high purity much later than the clusterings learned with single linkage clustering.  This is due to the inability of average linkage to account for the nonlinear and elongated shapes of these clusters.  Indeed, the opposite ends of the elongated cluster are quite far apart when measured with the average linkage metric, but are much closer when diffusion distances are used.  On the other hand, the bottleneck and Gaussian data illustrate how single linkage clusters may take a long time to achieve high purity, due to the fact that single linkage clustering is guided only by density, and is not robust to adversarial paths of points connecting two otherwise well-separated clusters.  Compared to single linkage and average linkage clustering, the clusters learned by LUND are robust to geometric distortions, adversarial paths, and noise.  

\subsubsection{Comparison with LUND}
\label{subsubsec:ComparisonLUND}

The proposed LAND algorithm (Algorithm \ref{alg:LAND}) integrates an active learning criterion into the LUND algorithm (Algorithm \ref{alg:LUND}).  It has been shown that when the the classes of the data $X$ are sufficiently coherent and pairwise well-separated, LUND with a good choice of $t$ perfectly labels all data points \cite{Maggioni2018_Learning}.  The unsupervised LUND algorithm depends critically on $t$, and the robustness of LUND to this choice of parameter suggests its usefulness.  However, developing practical methods for estimating a good choice of $t$ may be challenging in data that admits hierarchical cluster structure.  Indeed, consider the data in Figure \ref{fig:HierarchicalClusteredData}.  For this data, it is ambiguous whether there are two or four clusters.  Indeed, as shown in Figure \ref{fig:HierarchicalClusteredData} (c), if $\log_{10}(t)\in [1,3]$, LUND estimates there are 4 clusters.  If the time parameter satisfies $\log_{10}(t)\in [4,6]$, LUND estimates there are 2 clusters.  This is a fundamental ambiguity in unsupervised clustering, and one can view the ability of hierarchical clustering algorithms, and of LUND (depending on the time scale $t$) to detect the different possibilities for the number of clusters as a strength. Partial supervision allows for disambiguation in these situations.

\begin{figure}[!htb]
	\centering
	\begin{subfigure}[t]{.32\textwidth}
		\includegraphics[width=\textwidth]{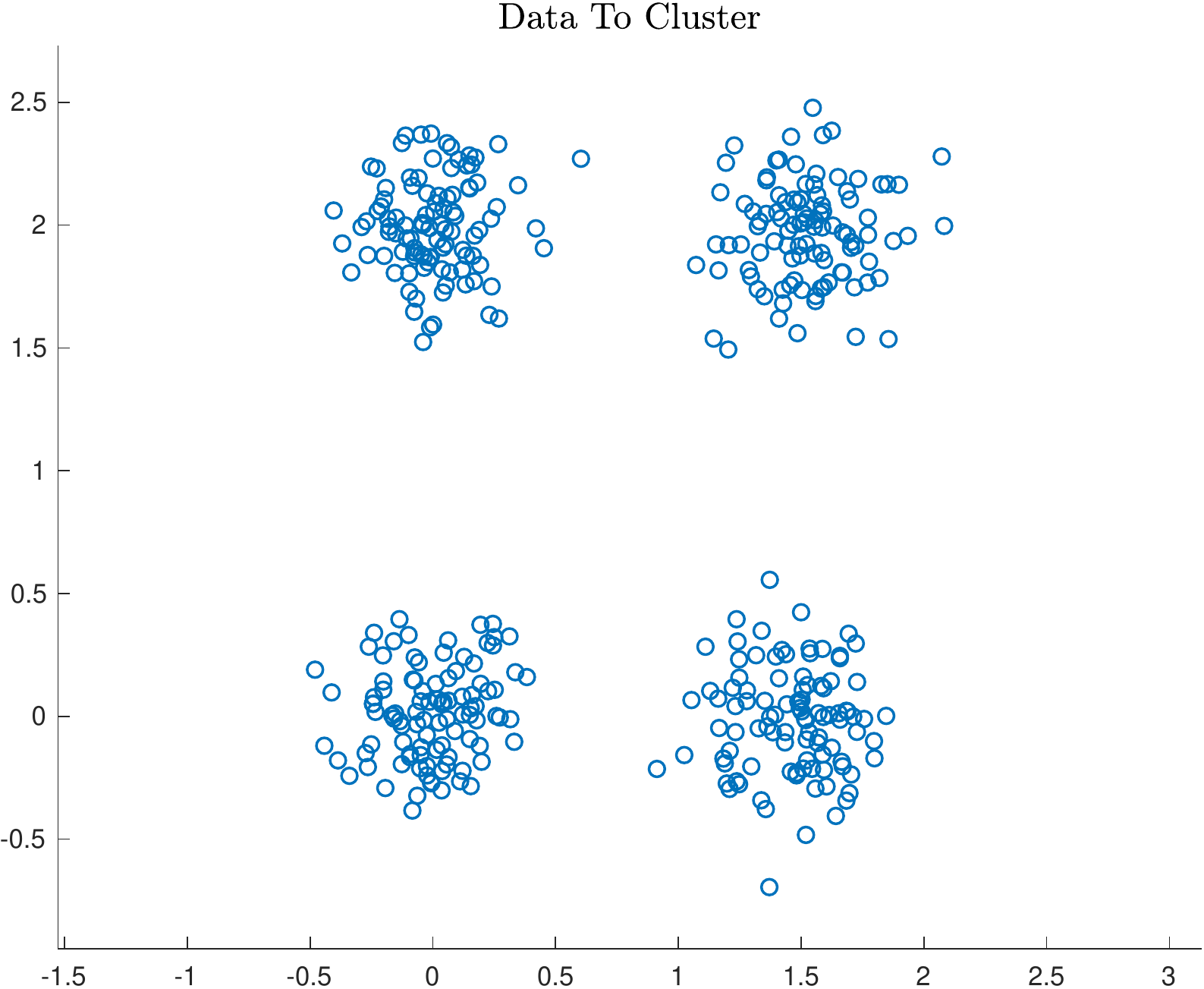}
		\subcaption{Hierarchical clustered data}
	\end{subfigure}
	\begin{subfigure}[t]{.32\textwidth}
		\includegraphics[width=\textwidth]{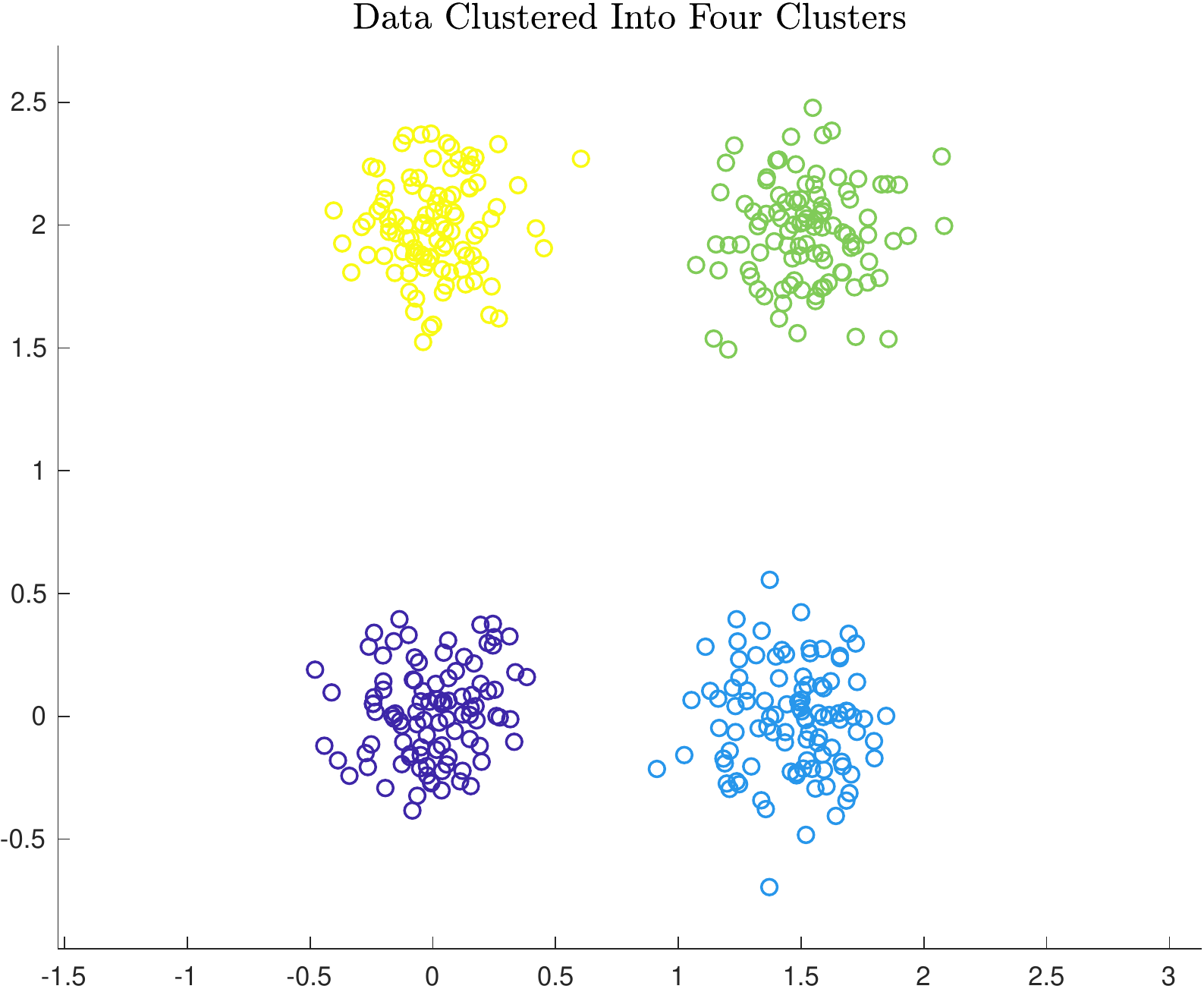}
		\subcaption{Labels with four classes}
	\end{subfigure}
	\begin{subfigure}[t]{.32\textwidth}
		\includegraphics[width=\textwidth]{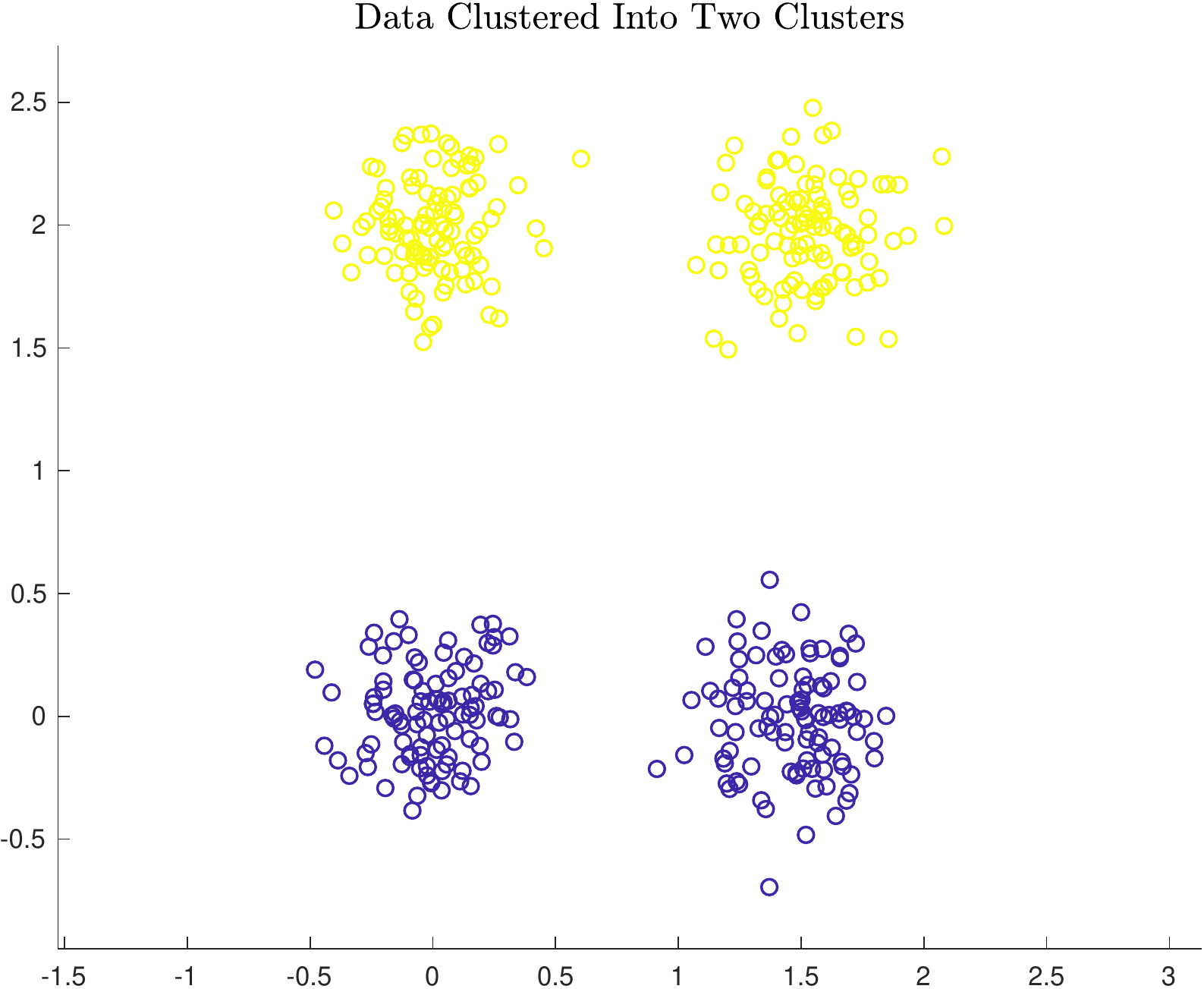}
		\subcaption{Labels with two classes}
	\end{subfigure}
	\caption{In (a), data with natural hierarchical structure is exhibited.  The four Gaussians have means $(0,0), (0,2), \left(\frac{3}{2},0\right), \left(\frac{3}{2},2\right)$.  While at one level of granularity there are 4 clusters (shown in (b)), at a coarser level of granularity the top 2 and bottom 2 Gaussians form clusters, leading to a clustering with only 2 clusters (shown in (c)).\label{fig:HierarchicalClusteredData}}		
\end{figure}

\begin{figure}[!htb]
	\centering
	\begin{subfigure}[t]{.32\textwidth}
		\includegraphics[width=\textwidth]{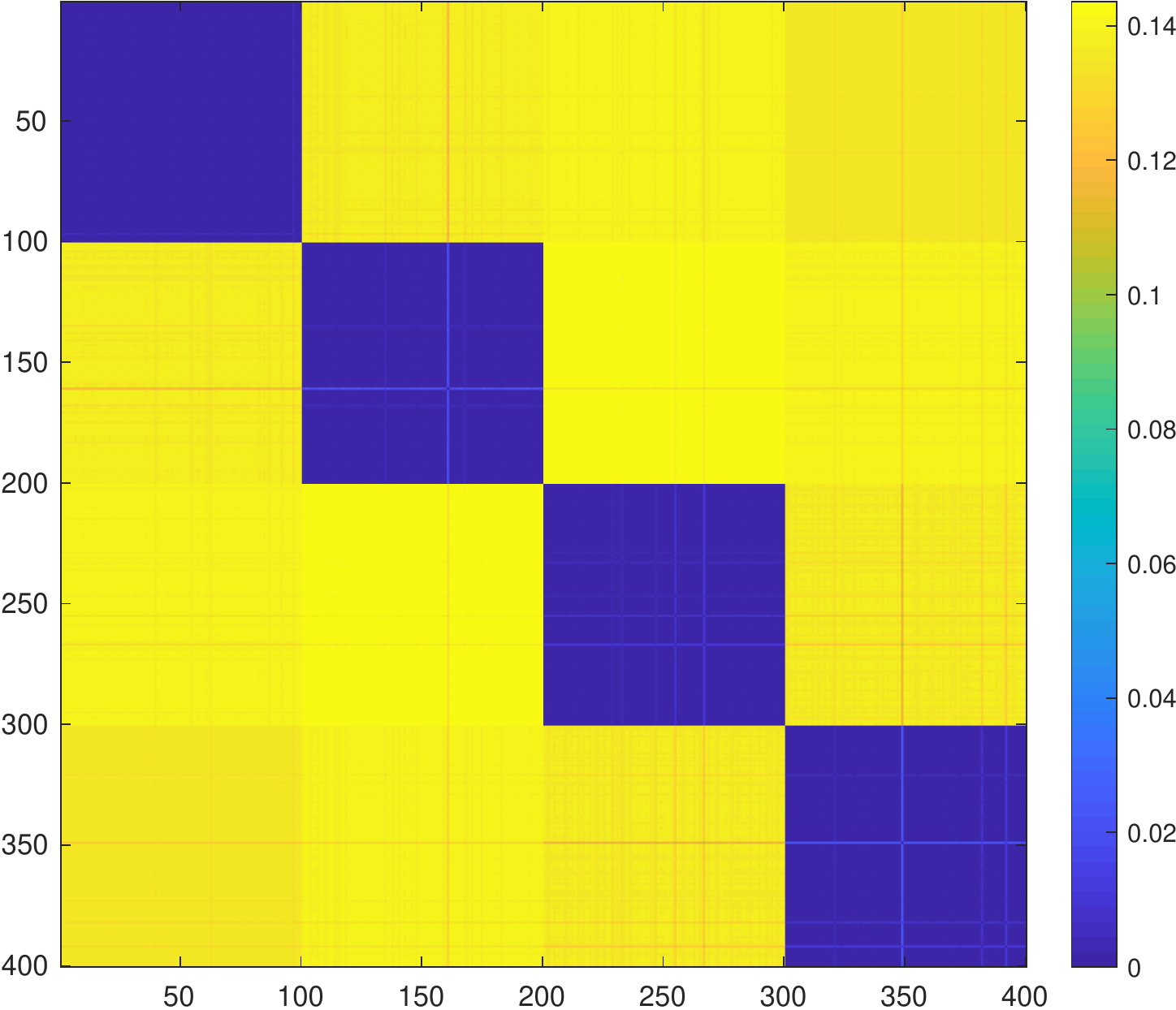}
		\subcaption{Diffusion distance matrix, $\log_{10}(t)=1.5$}
	\end{subfigure}
	\begin{subfigure}[t]{.32\textwidth}
		\includegraphics[width=\textwidth]{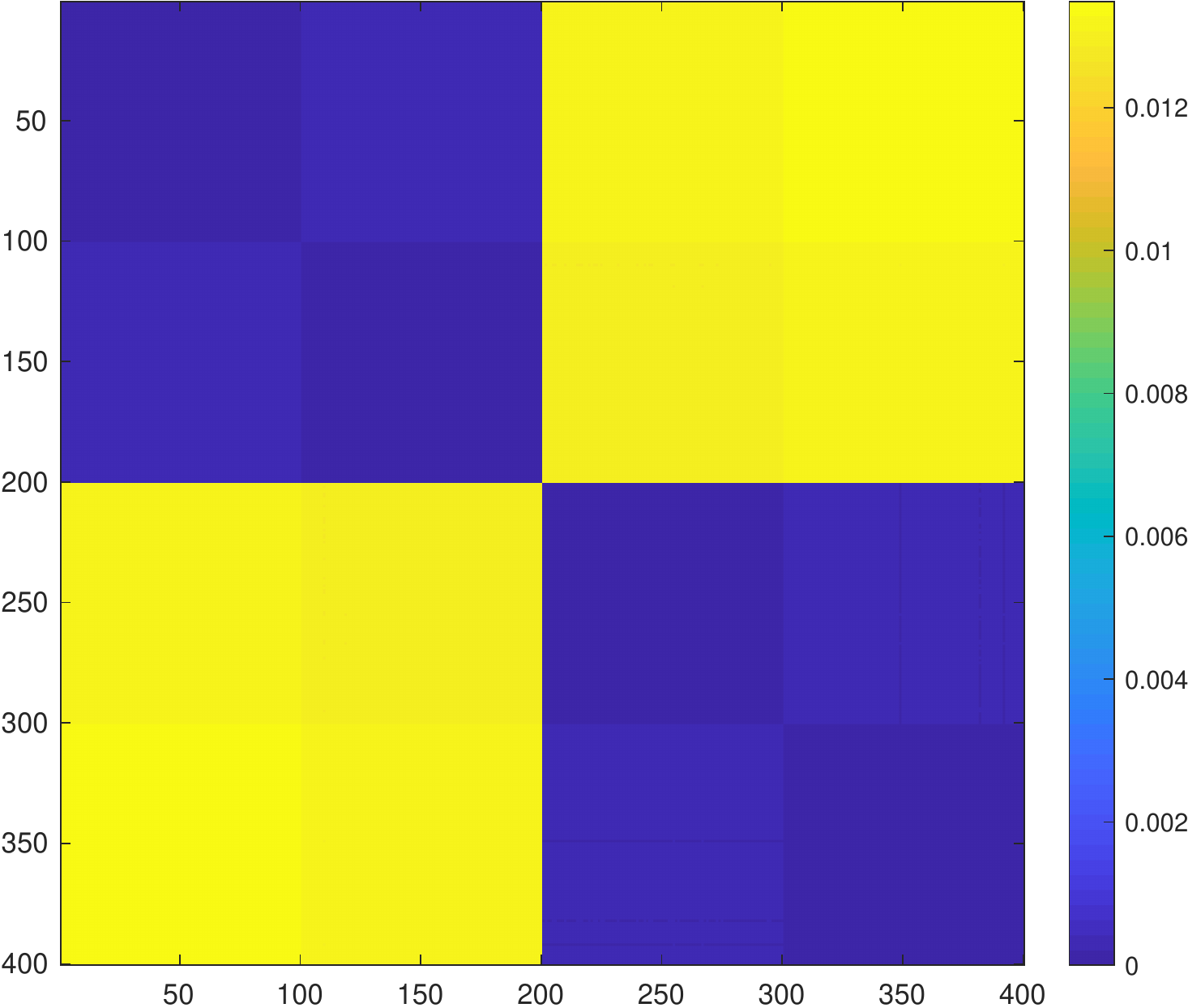}
		\subcaption{Diffusion distance matrix, $\log_{10}(t)=5$}
	\end{subfigure}
	\begin{subfigure}[t]{.32\textwidth}
		\includegraphics[width=\textwidth]{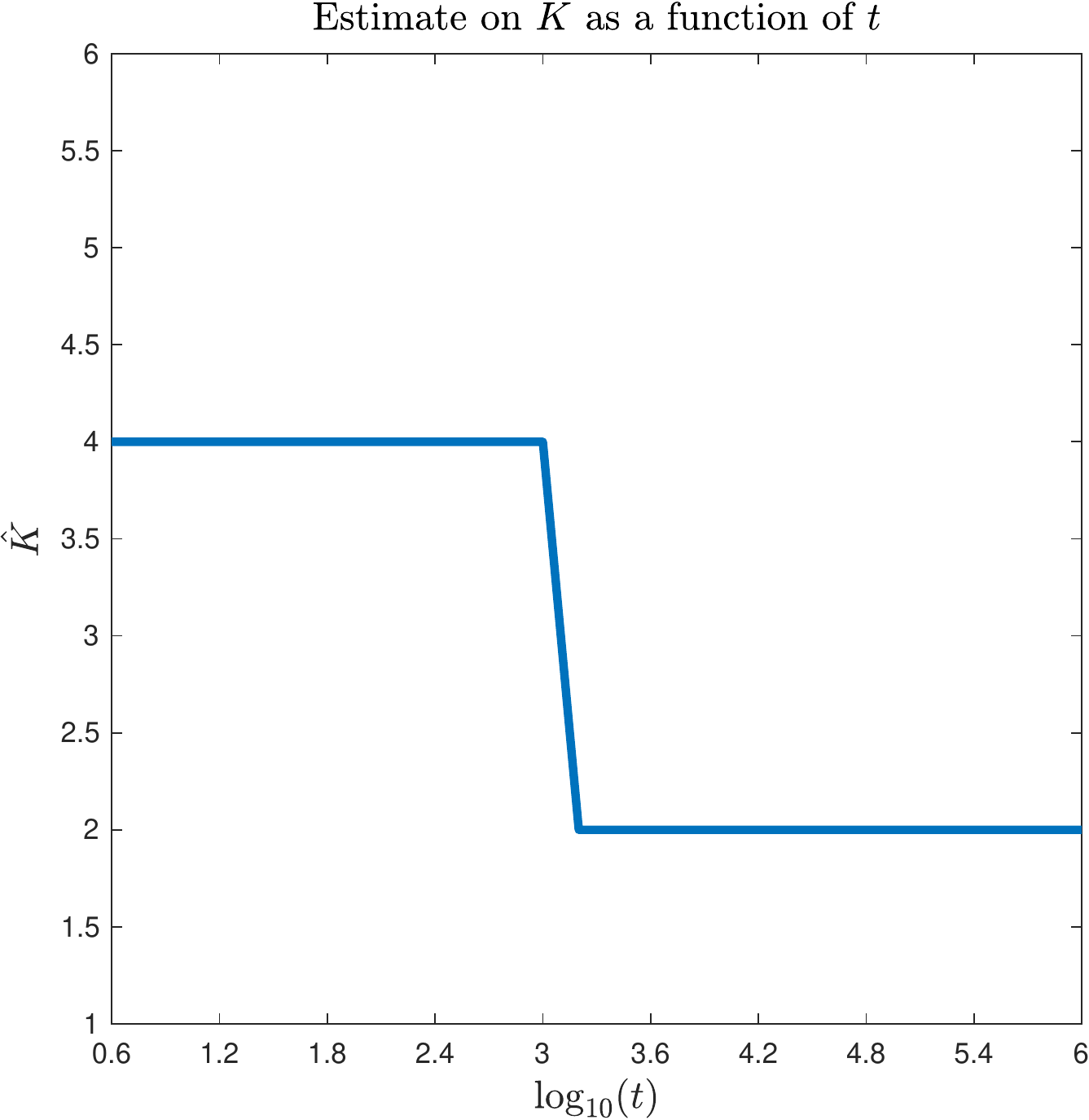}
		\subcaption{Estimates on $K$}
	\end{subfigure}
	\caption{The matrix of pairwise diffusion distances for $\log_{10}(t)=1.5$ and $\log_{10}(t)=5$ are shown in (a) and (b), respectively, illustrating the hierarchical cluster structure in the data.  This hierarchical structure introduces ambiguities into the estimation of the number of clusters $K$ in LUND, as shown (c).  For small time, $\hat{K}=4$, while for larger time $\hat{K}=2$.\label{fig:HierarchicalClusteredData_LUND}}
\end{figure}

Indeed, with a very small (4) labeled queries, LAND is able to overcome this obstacle and determine the labels of the data.  This is because even for large $t$, the top four values of $\Dt$ correspond to the modes of the four Gaussian clusters, and the diffusion distances within these clusters are quite small.  In the unsupervised case, for large $t$, the gap between the within-cluster and between-cluster distances for the four clusters are dwarfed by the the gap between the within-cluster and between-cluster distances for the two clusters, leading to ambiguity.  That is, when the underlying data is grouped into 2 clusters, $\Din/\Dbtw$ is large for large $t$ and small for small $t$; when the underlying data is grouped into 4 clusters, $\Din/\Dbtw$ is large for small $t$ and small for large $t$.  These lead to inherent ambiguity in how to choose $t$ in a fully unsupervised manner.  However, by bringing in just 4 labels, LAND is able to correctly label the dataset for both large and small $t$ values, as can be seen from Figure \ref{fig:HierarchicalClusteredData_LAND}.  In this sense, LAND introduces \emph{robustness to the time parameter} that may be problematic in LUND, at the cost of querying a small number of points.

\begin{figure}[!htb]
	\centering
	\begin{subfigure}[t]{.49\textwidth}
		\includegraphics[width=\textwidth]{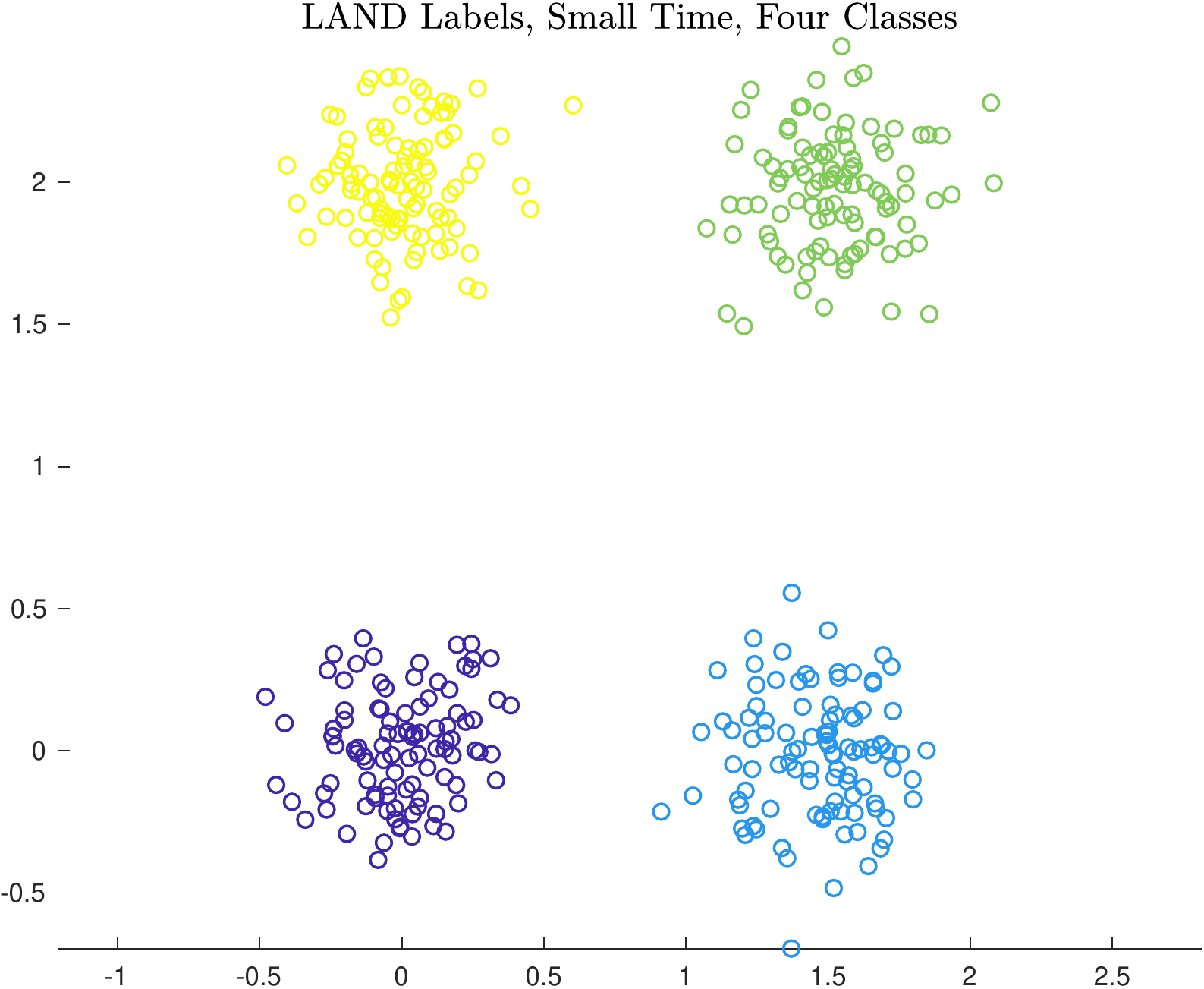}
		\subcaption{LAND, $\log_{10}(t)=2$, Four Clusters}
	\end{subfigure}
		\begin{subfigure}[t]{.49\textwidth}
		\includegraphics[width=\textwidth]{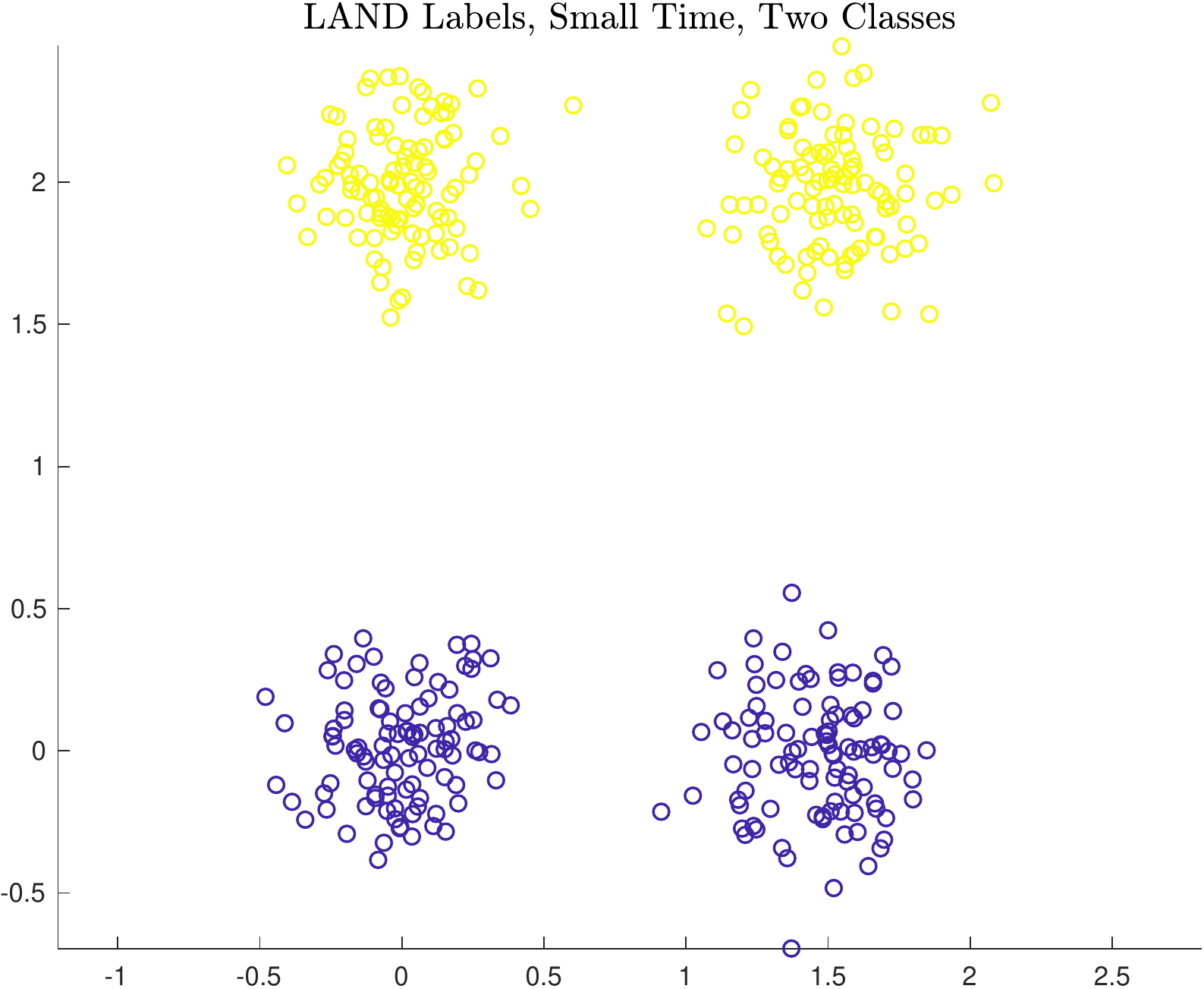}
		\subcaption{LAND, $\log_{10}(t)=2$, Two Clusters}
	\end{subfigure}
		\begin{subfigure}[t]{.49\textwidth}
		\includegraphics[width=\textwidth]{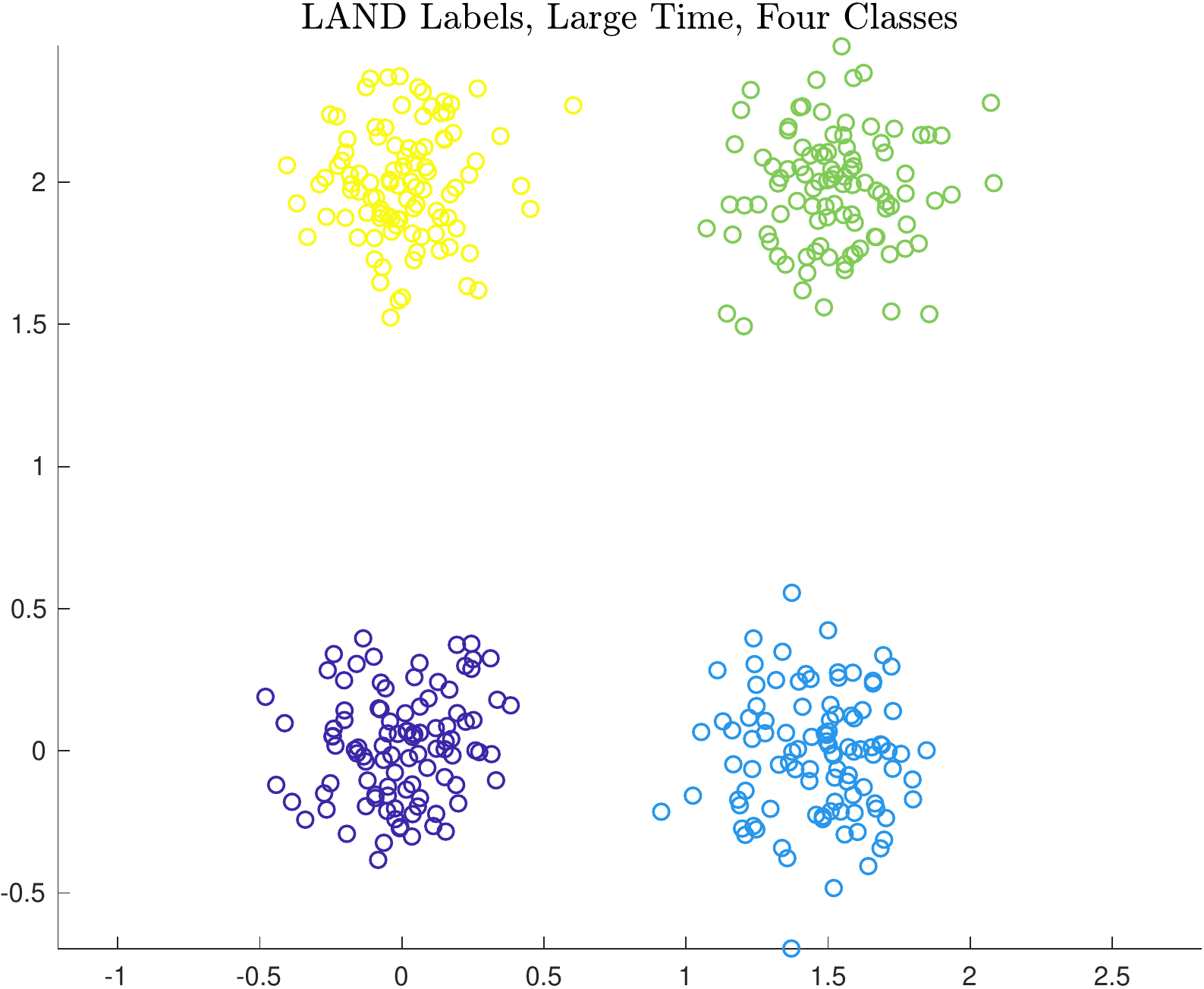}
		\subcaption{LAND, $\log_{10}(t)=5$, Four Clusters}
	\end{subfigure}
		\begin{subfigure}[t]{.49\textwidth}
		\includegraphics[width=\textwidth]{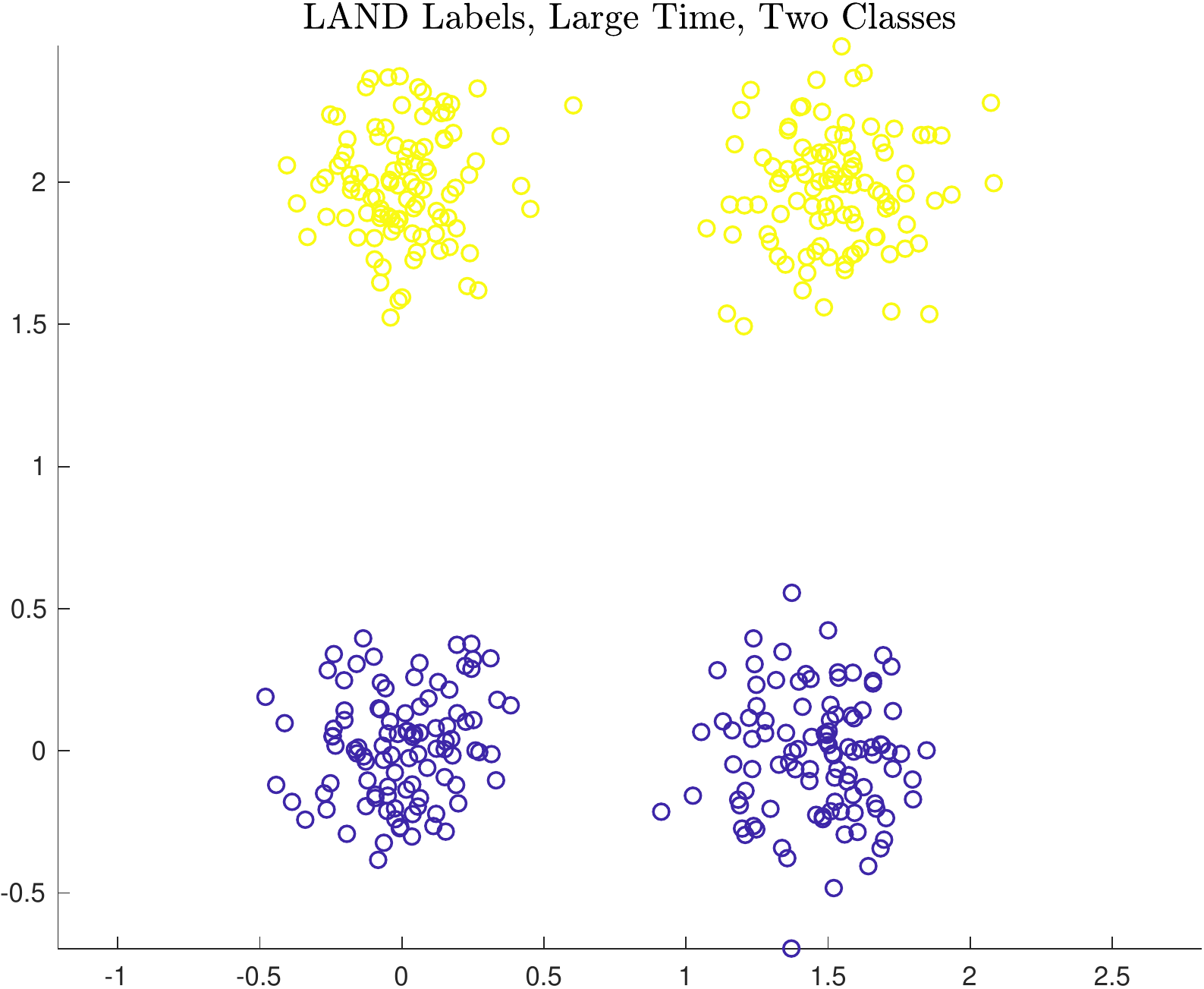}
		\subcaption{LAND, $\log_{10}(t)=5$, Two Clusters}
	\end{subfigure}
	
		\caption{LAND labelings of the data with four queries under two different scenarios: small diffusion time (top row) and large diffusion time (bottom row), and four latent clusters (first column) and two latent clusters (second column).  The clusters are closer in the horizontal direction than in the vertical direction, from whence the hierarchical structure is derived.  In all cases, LAND is able to to correctly label the data with just four queries, one for each Gaussian.\label{fig:HierarchicalClusteredData_LAND}}
\end{figure}

\subsection{Computational Complexity and Implementation}

The proposed Algorithm \ref{alg:LAND} has computational complexity depending crucially on the number of data points to label ($n$), the ambient dimensionality of the data ($D$), and the intrinsic dimensionality of the data ($d$).  

\begin{thm}Let $\{x_{i}\}_{i=1}^{n}\subset\mathbb{R}^{D}$ be data to label.  Suppose all except for $O(\log(n))$ points have a higher density point within its $O(\log(n))$ $D_{t}$-nearest neighbors.  In the case that a $\kNN$-sparse matrix $P$ is used, the LAND algorithm has complexity $O(C_{\text{NN}}+n\kNN+n\log(n)))$, where $C_{\text{NN}}$ is the cost of computing all $\kNN$ nearest neighbors.  
\end{thm}
\begin{proof}
The construction of the Markov transition matrix $P$ has complexity $O(C_{\text{NN}})$.  The subsequent kernel density estimation for all points is then $O(n\kNN)$.  The computation of $\rho_{t}$ for all points is $O(n\log(n))$, where we assume that all except for $O(\log(n))$ points has a higher density point within its $O(\log(n))$ $D_{t}$-nearest neighbors.  To estimate the modes from $\Dt$ requires sorting $n$ values, so has complexity $O(n\log(n))$.  Once the modes are estimated, labeling all points has complexity $O(n\log(n))$ by the assumption that all except for $O(\log(n))$ points has a higher density point within its $O(\log(n))$ $D_{t}$-nearest neighbors.  The result follows.  
\end{proof}

In the worst case, $C_{\text{NN}}=n^{2}$, so that LAND has quadratic complexity in $n$.  When the data has intrinsically low-dimensional structure, fast nearest neighbor searches reduce this complexity to be quasilinear in $n$.

\begin{cor}\label{cor:CC}Let $\{x_{i}\}_{i=1}^{n}\subset\mathbb{R}^{D}$ be data to label.  When the underlying data is intrinsically $d$-dimensional structure (in the sense of doubling dimension) and when $\kNN \ll \log(n)$, LAND has computational complexity $O(DC^{d}n\log(n)^{2}).$
\end{cor}
\begin{proof}
In the case that the data has intrinsically low-dimensional structure in the sense of doubling dimension, the cover tree indexing structure \cite{Beygelzimer2006_Cover} may be used so that to compute each points $\kNN$ has complexity $O(DC^{d}\kNN n\log(n))$.  The result follows.
\end{proof}

Corollary \ref{cor:CC} suggests that the proposed algorithm is appropriate for large numbers of data points $n$ in high dimension, provided that the intrinsic dimensionality of the data is small.

\section{Experimental Analysis}
\label{sec:Experiments}

We perform experiments on three representative synthetic datasets, as well as two real hyperspectral images \footnote{\url{http://www.ehu.eus/ccwintco/index.php/Hyperspectral_Remote_Sensing_Scenes}}.  Comparisons are made between LAND and two related methods:

\begin{enumerate}

\item \emph{LAND with random query points.}  This algorithm consists of Algorithm \ref{alg:LAND}, but with random points selected for querying, rather than the maximizers of $\Dt$.  Comparison with LAND will suggest if the query points determined by diffusion geometry and density---as captured by $\Dt$---are actually of significant value.  

\item \emph{Cluster-based active learning (CBAL)}. This algorithm \cite{Dasgupta2008_Hierarchical} is implemented using a hierarchical tree constructed from average linkage clustering.

\end{enumerate}

Three performance metrics are used to compare the active learning results.  \emph{Overall accuracy (OA)} is the ratio of correctly labeled pixels to the total number of pixels.  \emph{Average accuracy (AA)} averages the OA of each class, equalizing the significance of small and large classes.  \emph{Cohen's }$\kappa$-\emph{statistic} ($\kappa$) is a measure of agreement between two labelings that is robust to random chance \cite{Cohen1960_Coefficient}.

\subsection{Experiments on Synthetic Data}
 
Experimental results on the three synthetic datasets introduced in Figure \ref{fig:HierarchicalPurities} are shown in Figure \ref{fig:SyntheticHSI_Results}, illustrating the efficacy of LAND.  In all cases, LAND achieves near perfect accuracy with fewer than 10 labels, while the comparison methods converge to high accuracy much more gradually.  

\begin{figure}[!htb]
	\centering
	\begin{subfigure}[t]{.32\textwidth}
		\includegraphics[width=\textwidth]{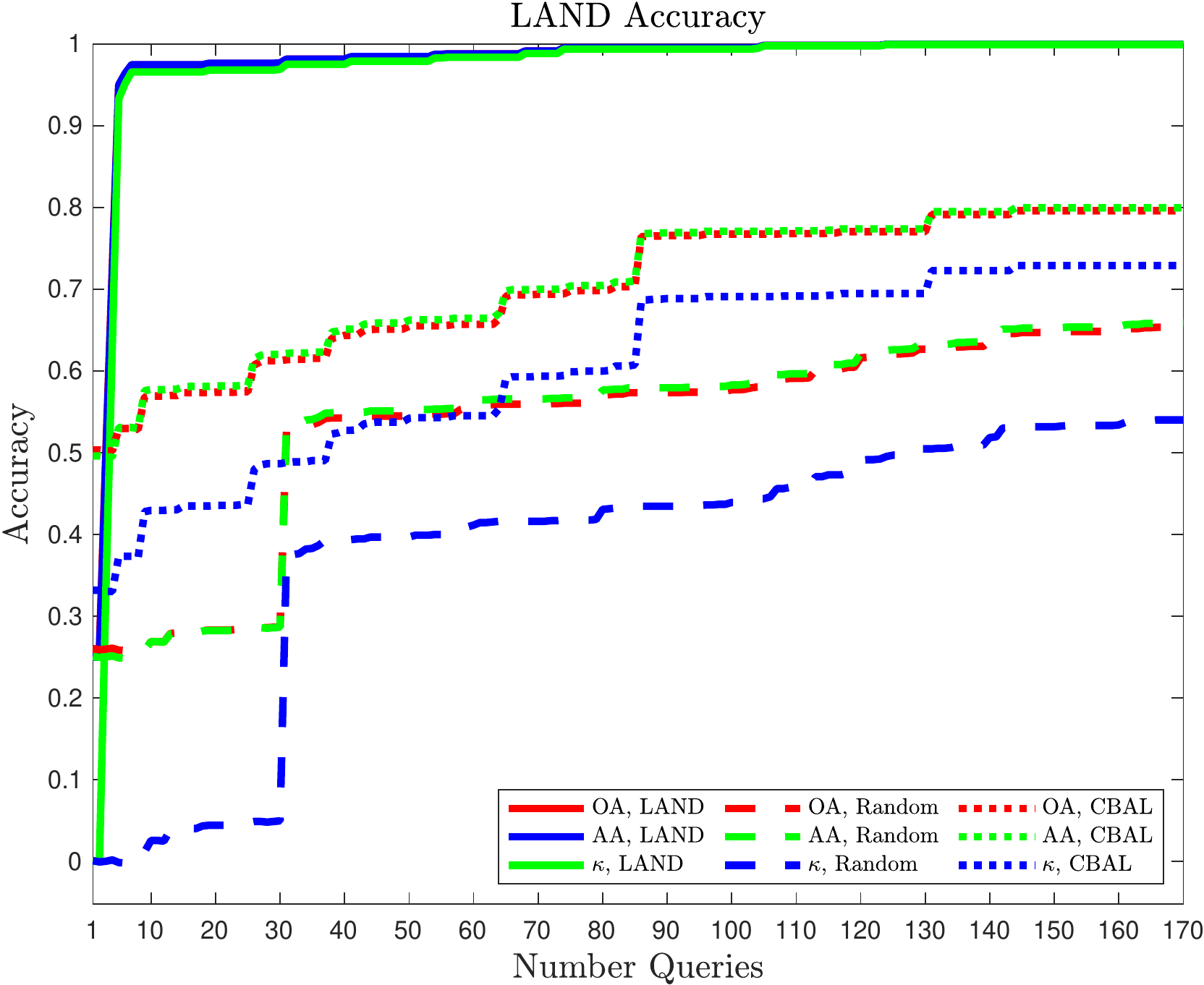}
		\subcaption{Bottleneck data results}
	\end{subfigure}
	\begin{subfigure}[t]{.32\textwidth}
		\includegraphics[width=\textwidth]{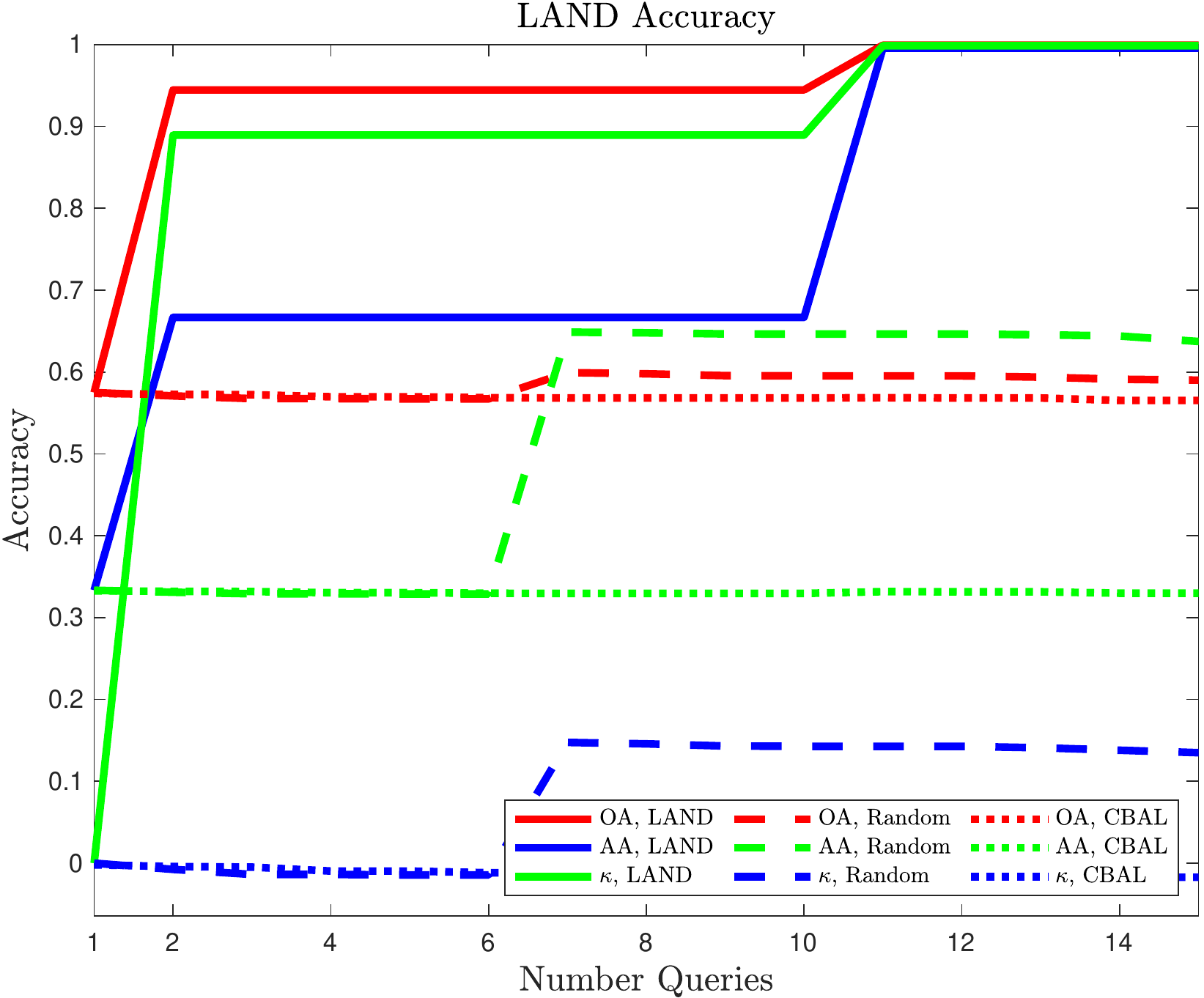}
		\subcaption{Geometric data results}
	\end{subfigure}
	\begin{subfigure}[t]{.32\textwidth}
		\includegraphics[width=\textwidth]{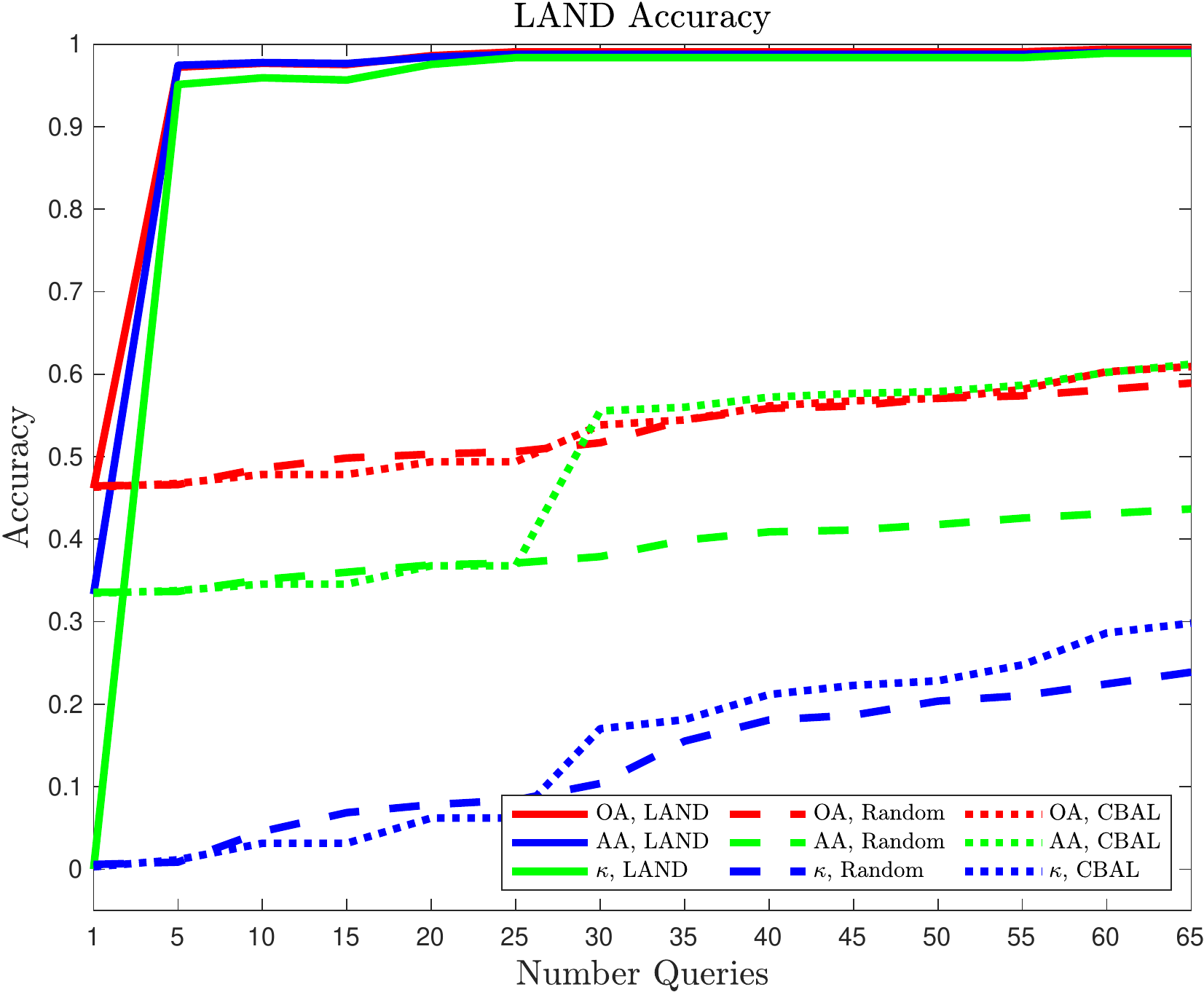}
		\subcaption{Gaussian data results}
	\end{subfigure}
		\caption{\label{fig:SyntheticHSI_Results}Experimental results on the synthetic datasets introduced in Figure \ref{fig:HierarchicalPurities}.  We see that LAND achieves rapid convergence to perfect labeling accuracy, compared to much slower convergence for the two comparison methods.}
\end{figure}

\subsection{Experiments on Hyperspectral Data}

In order to illustrate the efficacy of LAND on real data, we demonstrate its performance on \emph{hyperspectral imagery}, which constitutes an important data type in the remote sensing of the environment \cite{Chang2003_Hyperspectral}.  An HSI is an image consisting of $D$ spectral bands, each localized to a narrow electromagnetic range.  The concatenation of these $D$ spectral bands provides highly detailed information about the materials being imaged, and can allow for precise discrimination of specific objects in the scene.  While nominally a 3-dimensional tensor, an HSI is often analyzed by collapsing the spatial coordinates to produce a dataset $\{x_{i}\}_{i=1}^{n}\subset\mathbb{R}^{D}$, where $n$ is the total number of pixels in the image and $D$ is the total number of spectral bands.  When large training sets of labeled pixels are available, classification of an HSI scene may be effectively performed using a range of techniques, including support vector machines \cite{Melgani2004_Classification}, deep learning \cite{Chen2014_Deep}, and random forests \cite{Ham2005_Investigation}.

Traditional supervised learning has led to strong empirical performance for HSI classification.  However, supervised learning for HSI---particularly state-of-the-art deep learning---is predicated on the availability of large labeled training sets, which must be collected and annotated, typically by human experts.  The need for large training sets is exacerbated by the high dimensionality of the data.  The collection of large training sets may not be practical in the context of HSI, where there is a huge number of possible classes and large variabilities are introduced by sensing conditions.  Indeed, the task of generating huge training sets for general HSI imagery is quite onerous, and may even require the deployment of humans to observe physically the scene that has been remotely sensed, which is very resource intensive.  It is thus crucial to develop methods that can label HSI with no labeled training data \cite{Acito2003, Paoli2009, Cahill2014_Schroedinger, Cariou2015, Gillis2015, Zhang2016, Chen2017, Zhu2017, Zhai2017_New, Murphy2019_Unsupervised, Murphy2019_Spectral} or a combination of labeled and unlabeled data \cite{Camps2007_Semi, Ratle2010, Li2013_1}.

Active learning for HSI is an important method for achieving high-accuracy classification results, without requiring large labeled training sets \cite{Rajan2008_Active, Tuia2009_Active, Li2010_Semisupervised, Sun2015_Active, Zhang2016_Active, Murphy2018_Iterative}.  These methods typically query for labels points near the boundaries of classes, thus improving the convergence of the learning algorithm towards a good classifier.  LAND, on the other hand, exploits cluster structure in the data.  

\subsubsection{Experimental Results for HSI}
 
We perform active learning experiments on two real HSI datasets, shown in Figure \ref{fig:SalinasA_Data} and \ref{fig:Pavia_Data}, respectively.  

\begin{figure}[!htb]
	\centering
	\begin{subfigure}[t]{.49\textwidth}
		\includegraphics[width=\textwidth]{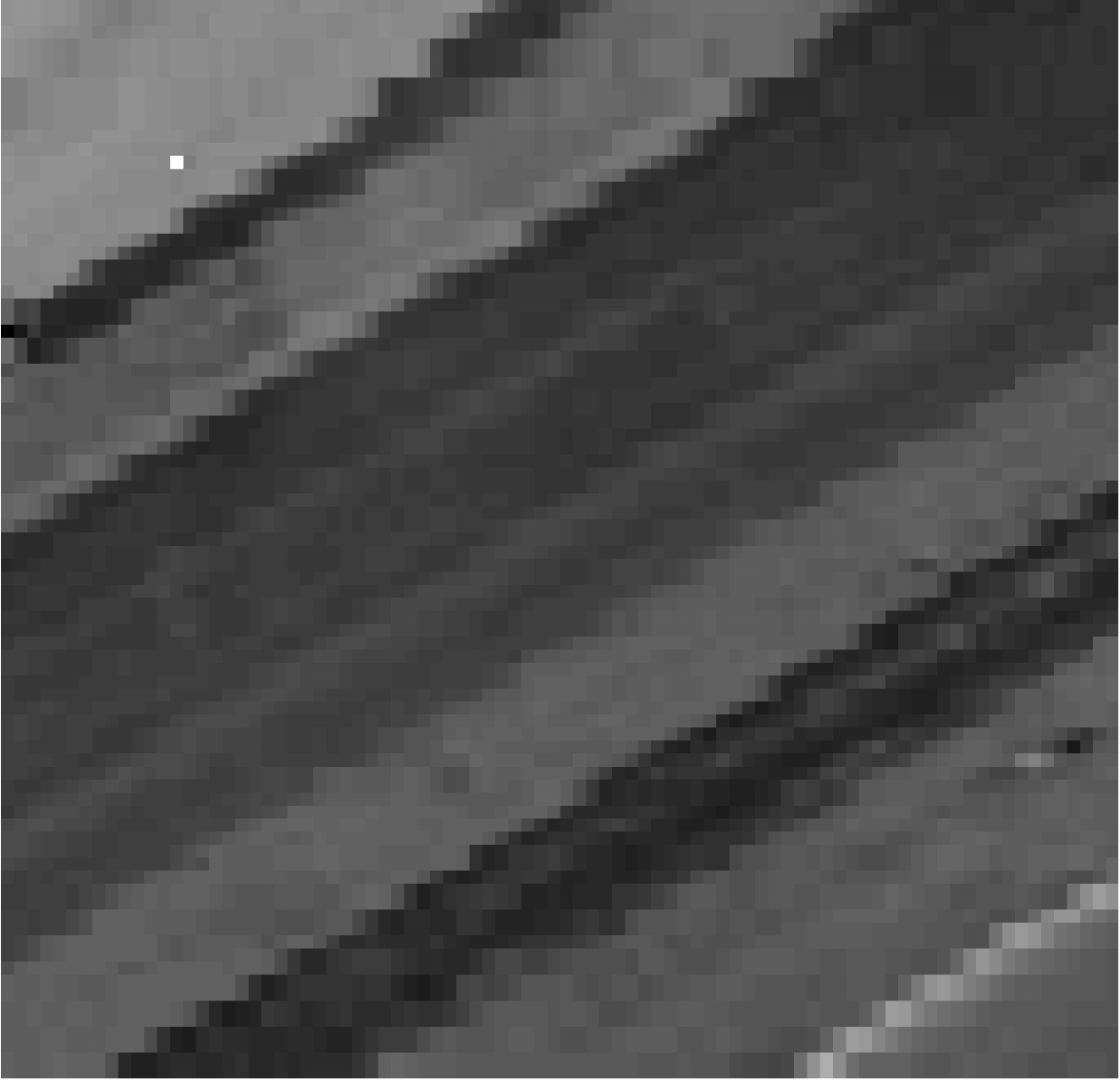}
		\subcaption{Salinas A}
	\end{subfigure}
	\begin{subfigure}[t]{.49\textwidth}
		\includegraphics[width=\textwidth]{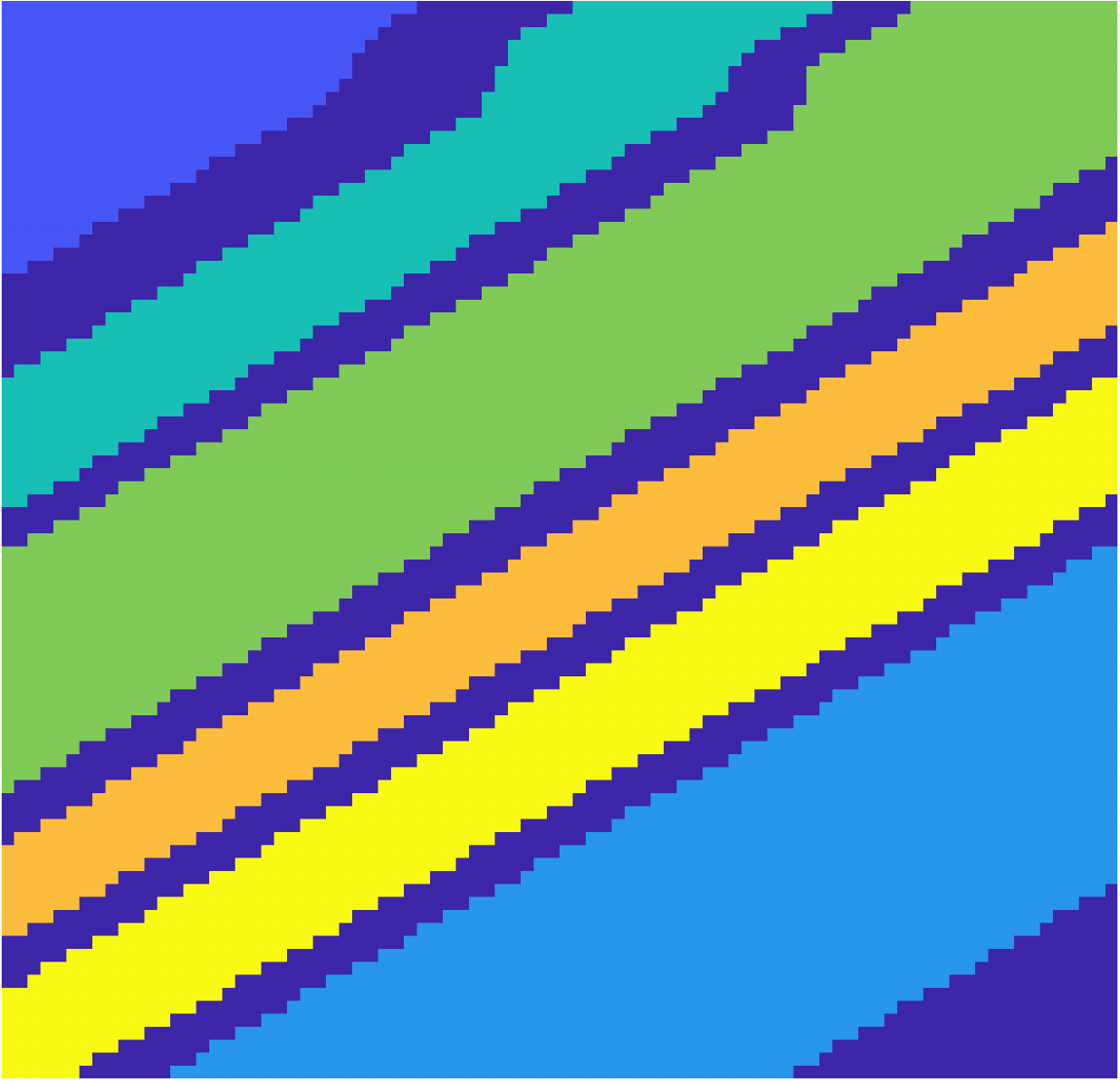}
		\subcaption{Ground Truth}
	\end{subfigure}
		\caption{\label{fig:SalinasA_Data}  The Salinas A dataset consists of $83 \times 86 = 7138$ pixels in $D=224$ dimensions.  The image has spatial resolution $3.7$m/pixel, and was recorded over Salinas, USA by the Aviris sensor.  The six labelled classes are arranged in diagonal rows, and are quite spatially regular.  The sum across all spectral bands is shown in (a), and the labeled ground truth is shown in (b), with pixels having the same class being given the same color.}
\end{figure}

\begin{figure}[!htb]
	\centering
	\begin{subfigure}[t]{.49\textwidth}
		\includegraphics[width=\textwidth]{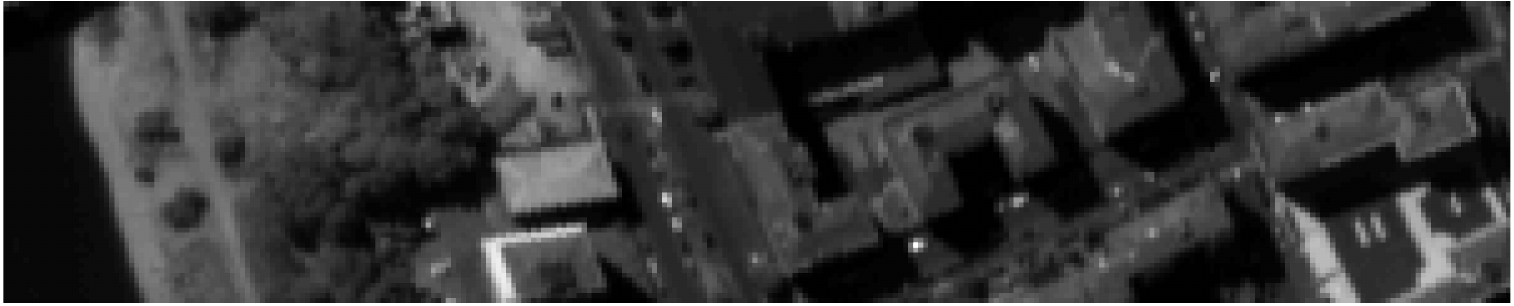}
		\subcaption{Pavia}
	\end{subfigure}
	\begin{subfigure}[t]{.49\textwidth}
		\includegraphics[width=\textwidth]{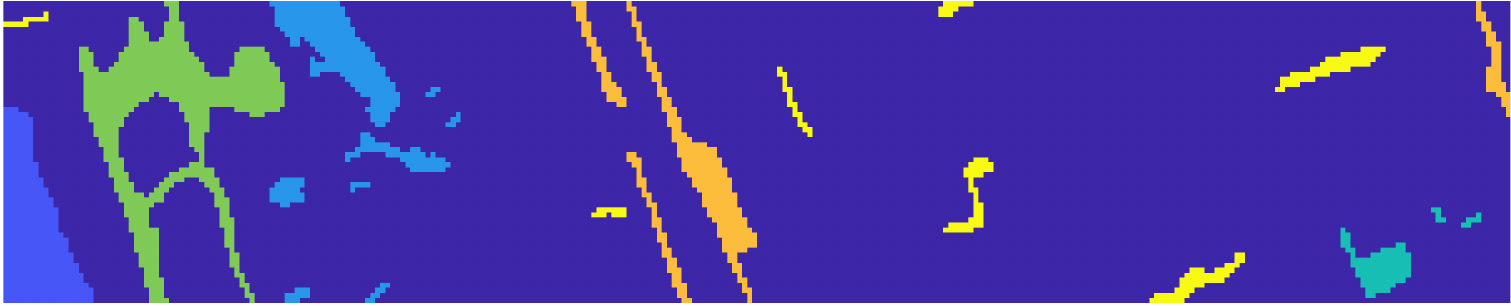}
		\subcaption{Ground Truth}
	\end{subfigure}
		\caption{\label{fig:Pavia_Data}  Pavia data consists of $270\times 50=13500$ subset of the full Pavia data set.  The image has spatial resolution $1.3$m/pixel, and was recorded over Pavia, Italy by the  ROSIS sensor.  It consists of 6 spatial classes, some of which are quite well-spread out in the image.  The sum across all spectral bands is shown in (a), and the labeled ground truth is shown in (b), with pixels having the same class being given the same color.}
\end{figure}

Experimental results for the three methods on the Salinas A and Pavia datasets are shown in Figure \ref{fig:RealHSI_Results}.  For the Salinas A dataset, accuracy with LAND is strong, with only 10 labels leading to highly accurate empirical results, and subsequent labels leading to rapid improvement towards perfect accuracy.  In particular, compared to using random query labels or CBAL, the improvement of LAND as a function of the number of queries is rapid.  For the Pavia dataset, there is a similar early jump in accuracy for LAND, while the improvement is slower for the comparison methods.

\begin{figure}[!htb]
	\centering
	\begin{subfigure}[t]{.49\textwidth}
		\includegraphics[width=\textwidth]{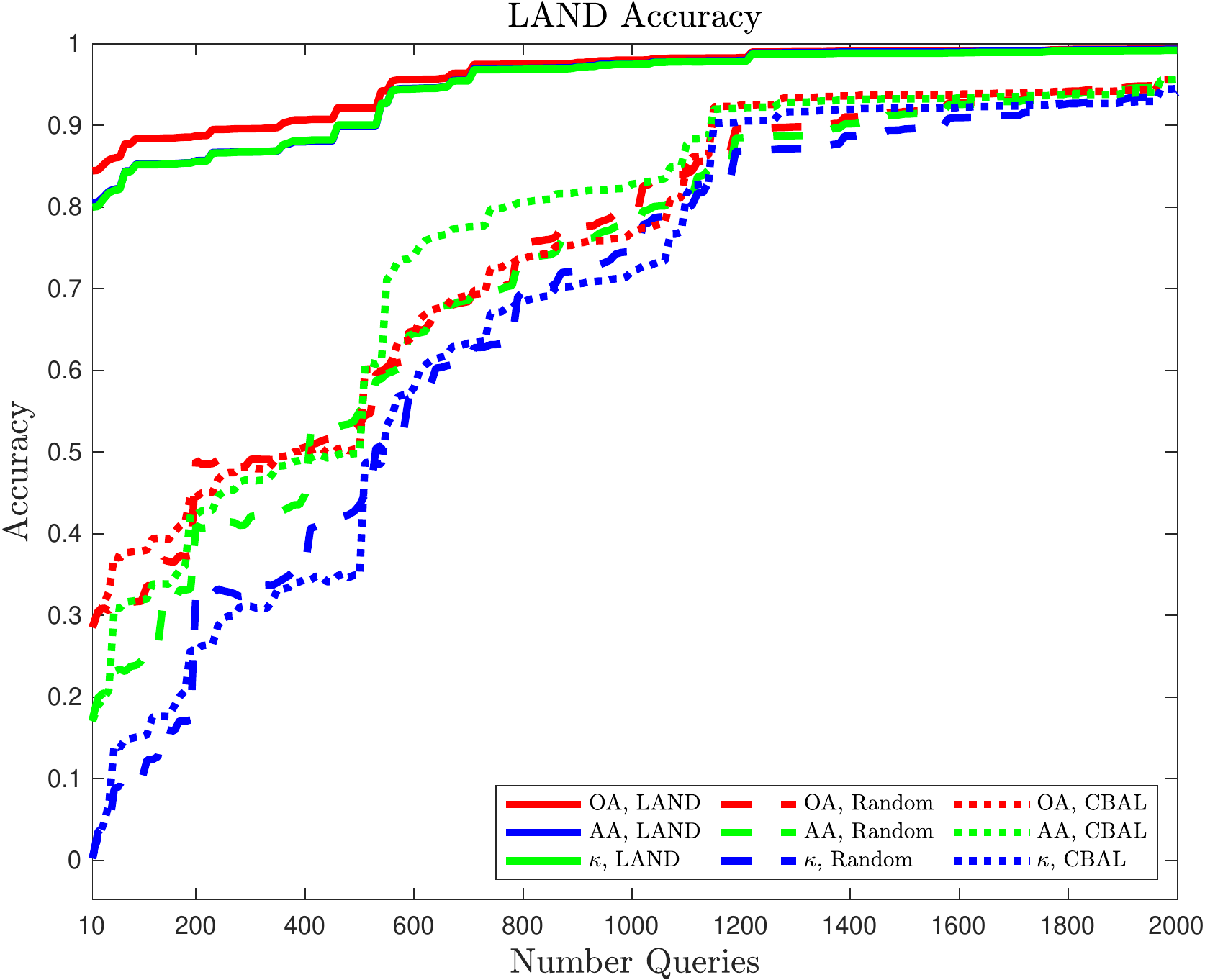}
		\subcaption{Salinas A results}
	\end{subfigure}
	\begin{subfigure}[t]{.49\textwidth}
		\includegraphics[width=\textwidth]{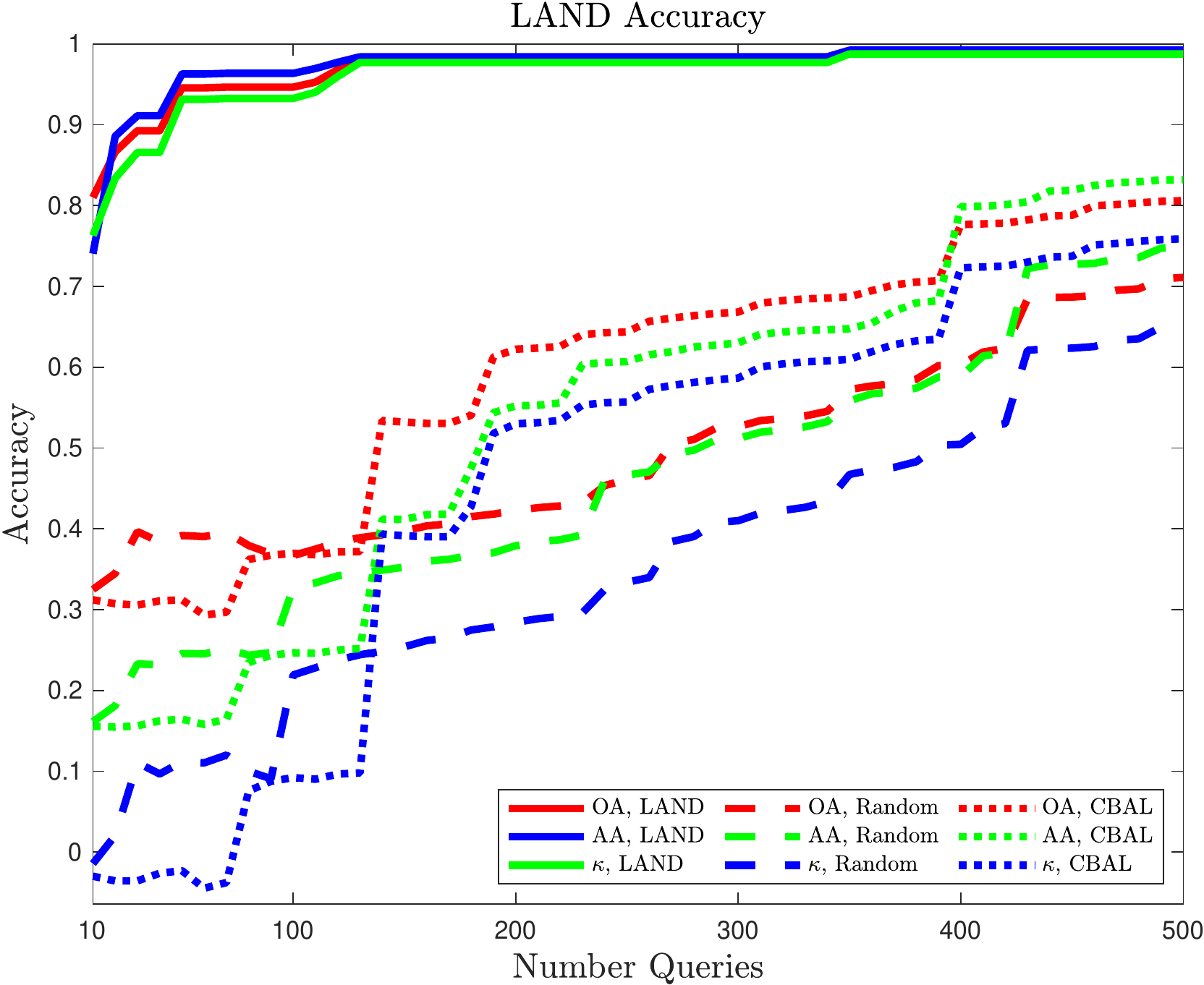}
		\subcaption{Pavia results}
	\end{subfigure}
		\caption{The active learning results for the Salinas A and Pavia datasets are shown in (a) and (b), respectively.  In both cases, the LAND algorithm strongly outperforms the modified LAND variant using randomly selected training data, and the CBAL algorithm.  In particular, LAND is able to achieve a significant improvement in accuracy with a very small number of labels.}
	\label{fig:RealHSI_Results}		
\end{figure}

\section{Conclusions and Future Work}
\label{sec:Conclusions}

The LAND algorithm integrates diffusion geometry and density estimation to efficiently estimate query points that are highly impactful on overall labeling accuracy in the active learning setting.  Our theoretical and empirical analyses show LAND's robustness to geometric distortions of the underlying data classes, and our experiments on real-world HSI demonstrate its effectiveness in accurately labeling high-dimensional datasets with a very small number of query points.  

In the context of HSI, developing active learning methods that incorporate spatial proximity into the underlying diffusion process is of interest.  This information may suggest that it is useful to query information in a spatially homogeneous region, where it can be most impactful.  The integration of spatial information into a variant of the LUND algorithm adapted for HSI has proven effective \cite{Murphy2018_Iterative, Murphy2019_Spectral}, and it is likely that such information would similarly boost the effectiveness of LAND.  

It is of interest to develop a cross-validation scheme that exploits the active learning queries in order to iteratively update the optimal choice of time parameter $t$.  Indeed, as argued in Section \ref{subsubsec:ComparisonLUND}, the use of a very small (essentially $O(\numclust)$) active learning queries can be used to achieve robustness to the parameter $t$, which is critically important in the LUND algorithm.  However, it may be possible to update the time parameter in an iterative fashion, by selecting at each time step a time scale that separates all the modes learned so far, before querying a new point.  This has the potential to require fewer queries to learn all the classes, since the parameter is being adaptively optimized at each time step, rather than after all queries have been made.

\bibliography{LAND.bib}
\bibliographystyle{unsrt}

\bigskip

\end{document}